%% file: main.tex
\documentclass[acmsmall]{acmart}

\usepackage{algorithm,algpseudocode}
\usepackage{amsthm}
\usepackage{bm}
\usepackage{mathtools}

\usepackage{float}
\usepackage{graphicx}
\usepackage{paralist}
\usepackage{multicol}
\usepackage{caption}
\usepackage{subcaption}

\DeclareMathOperator*{\argmin}{argmin}

\algnewcommand{\LineComment}[1]{\Statex \(\triangleright\) #1}

\AtBeginDocument{%
  \providecommand\BibTeX{{%
    \normalfont B\kern-0.5em{\scshape i\kern-0.25em b}\kern-0.8em\TeX}}}

\setcopyright{acmcopyright}
\copyrightyear{2020}
\acmYear{2020}
\acmDOI{10.1145/1122445.1122456}

\acmConference[Woodstock '18]{Woodstock '18: ACM Symposium on Neural
  Gaze Detection}{June 03--05, 2018}{Woodstock, NY}
\acmBooktitle{Woodstock '18: ACM Symposium on Neural Gaze Detection,
  June 03--05, 2018, Woodstock, NY}
\acmPrice{15.00}
\acmISBN{978-1-4503-XXXX-X/18/06}

\begin{document}

\newcounter{lemma}
\newcounter{corollary}
\newtheorem{lem}[lemma]{Lemma}
\newtheorem{thm}[theorem]{Theorem}
\newtheorem{cor}[corollary]{Corollary}

\title{Learning Manifolds from Non-stationary Streams}

\author{Suchismit Mahapatra}
\email{suchismi@buffalo.edu}
\author{Varun Chandola}
\email{chandola@buffalo.edu}
\affiliation{%
  \institution{\\Department of Computer Science and Engineering, University at Buffalo}
  \city{Buffalo}
  \state{New York}
  \postcode{14260-1660}
}
\renewcommand{\shortauthors}{Mahapatra and Chandola}

\begin{abstract}
  Streaming adaptations of manifold learning based dimensionality reduction methods, such as {\em Isomap}, are based on the assumption that a small initial batch of observations is enough for exact learning of the manifold, while remaining streaming data instances can be cheaply mapped to this manifold. However, there are no theoretical results to show that this core assumption is valid. Moreover, such methods typically assume that the underlying data distribution is stationary. Such methods are not equipped to detect, or handle, sudden changes or gradual drifts in the distribution that may occur when the data is streaming. We present theoretical results to show that the quality of a manifold asymptotically converges as the size of data increases. We then show that a Gaussian Process Regression (GPR) model, that uses a manifold-specific kernel function and is trained on an initial batch of sufficient size, can closely approximate the state-of-art streaming Isomap algorithms. The predictive variance obtained from the GPR prediction is then shown to be an effective detector of  changes in the underlying data distribution. Results on several synthetic and real data sets show that the resulting algorithm can effectively learn lower dimensional representation of high dimensional data in a streaming setting, while identifying shifts in the generative distribution.
\end{abstract}



\maketitle
\section{Introduction}\label{sec:introduction}
High-dimensional data is inherently difficult to explore and analyze, owing to the ``curse of dimensionality'' that render many statistical and machine learning techniques inadequate. In this context, {\em non-linear dimensionality reduction} (NLDR) has proved to be an indispensable tool. Manifold learning based NLDR methods, such as Isomap~\citep{tenenbaum2000}, Local Linear Embedding (LLE)~\citep{roweis2000}, etc., assume
that the distribution of the data in the high-dimensional observed space is not uniform. Instead, the data is assumed to lie near a non-linear low-dimensional manifold embedded in the high-dimensional space. By exploiting the geometric properties of the manifold, e.g., smoothness, such methods infer the low-dimensional representation of the data from the high-dimensional observations.


A key shortcoming of NLDR methods is their $O(n^3)$ complexity, where $n$ is the size of the data. If directly applied on streaming data, where data arrives one point at a time, NLDR methods have to recompute the entire manifold at every time step, making such a naive adaptation prohibitively expensive. To alleviate the computational problem, landmark-based methods~\citep{silva2003} or general out-of-sample extension methods~\citep{wu2004} have been proposed. However, these techniques are still computationally expensive for practical applications. Recently, a streaming adaptation of the Isomap algorithm~\citep{tenenbaum2000}, which is a widely used NLDR method, was proposed~\citep{schoeneman2017}. This method, called S-Isomap, relies on exact learning from a small initial batch of observations, followed by approximate mapping of subsequent stream of observations. An extension to the case when the observations are sampled from multiple, and possibly intersecting, manifolds, called S-Isomap++, was subsequently proposed~\citep{mahapatra2017}. 

Empirical results on benchmark data sets show that these methods can reliably learn the manifold with a small initial batch of observations. However two issues still remain. {\em First}, no theoretical bounds on the quality of the manifold, as a function of the initial batch size, exist. {\em Second}, these methods assume that the underlying generative distribution is {\em stationary} over the stream, and are unable to detect when the distribution ``drifts'' or abruptly ``shifts'' away from the base, resulting in incorrect low-dimensional mappings (See Figure~\ref{fig:expts}). 

\begin{figure}[!htbp]
\centering
\includegraphics[width=1.0\textwidth]{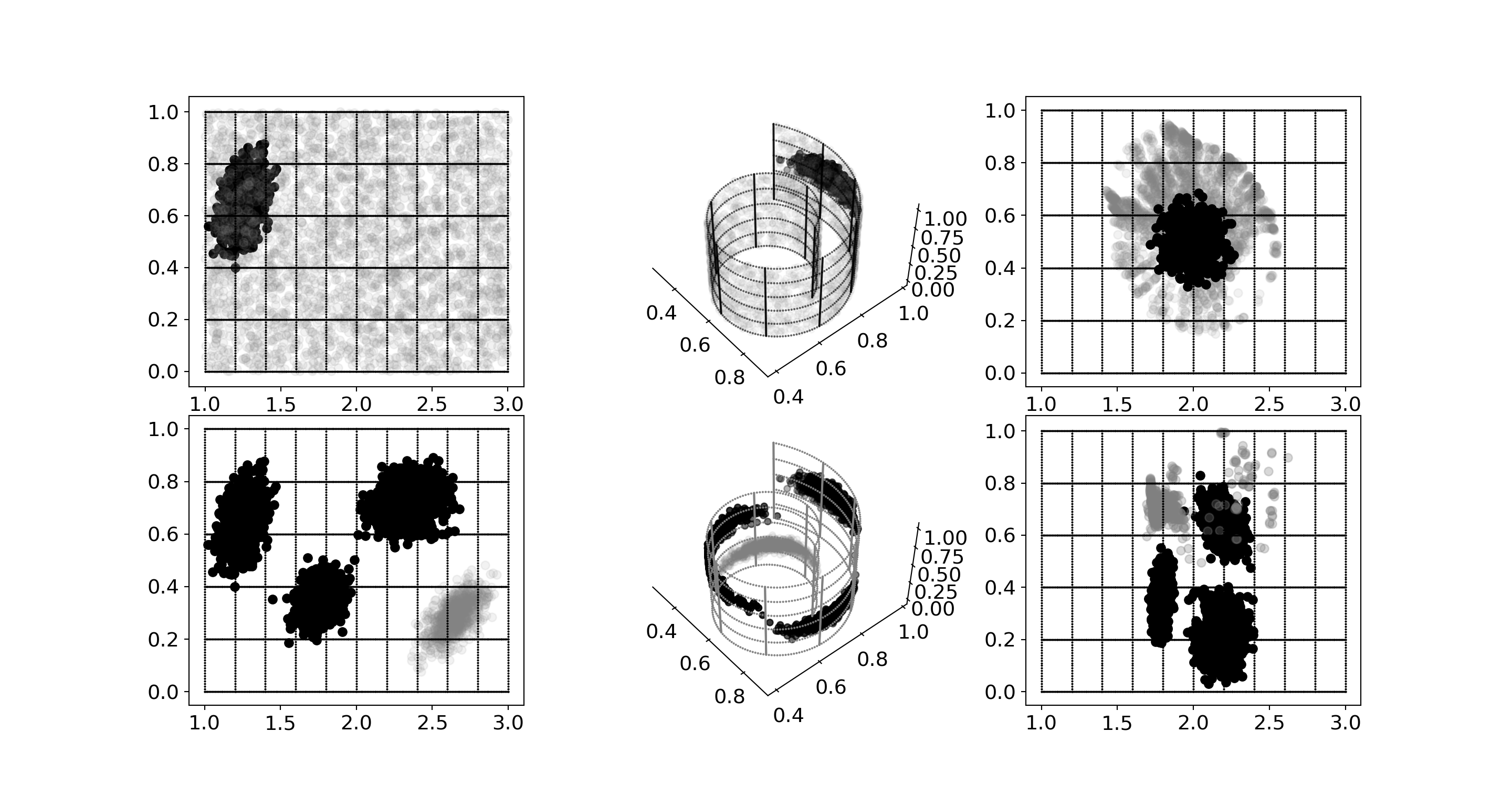}
\caption{Impact of changes in the data distribution on streaming NLDR. In the {\em top} panel, the true data lies on a 2D manifold ({\em top-left}) and the observed data is in $\mathbb{R}^3$ obtained by using the {\em swiss-roll} transformation of the 2D data ({\em top-middle}). The streaming algorithm (S-Isomap~\citep{schoeneman2017}) uses a batch of samples from a 2D Gaussian (black), and maps streaming points sampled from a uniform distribution (gray). The streaming algorithm performs well on mapping the batch points to $\mathbb{R}^2$ but fails on the streaming points that ``drift'' away from the batch ({\em top-right}). In the {\em bottom} panel, the streaming algorithm (S-Isomap++~\citep{mahapatra2017}) uses a batch of samples from three 2D Gaussians (black). The stream points are sampled from the three Gaussians and a new Gaussian (gray). The streaming algorithm performs well on mapping the batch points to $\mathbb{R}^2$ but fails on the streaming points that are ``shifted'' from the batch ({\em bottom-right}). Both streaming algorithms are discussed in Section~\ref{sec:preliminaries}.}
\label{fig:expts}
\end{figure}
The focus of this paper is two-fold. We first provide theoretical results that show that the quality\footnote{See Section~\ref{sec:preliminaries} for the definition of manifold quality} of a manifold, as learnt by Isomap, asymptotically converges as the data size, $n$, increases. This is a necessary result to show the correctness of streaming methods such as S-Isomap and S-Isomap++, under the assumption of stationarity. Next, we propose a methodology to detect changes in the underlying distribution of the stream properties (drifts and shifts), and inform the streaming methods to update the base manifold.

We employ a Gaussian Process (GP)~\citep{williams2001} based adaptation of Isomap to process high-throughput streams.  The use of GP is enabled by a kernel that measures the relationship between a pair of observations along the manifold, and not in the original high-dimensional space. We prove that the low-dimensional representations inferred using the GP based method -- {\em GP-Isomap} -- are equivalent to the representations obtained using the state-of-art streaming Isomap methods~\citep{schoeneman2017,mahapatra2017}. Additionally, we empirically show, on synthetic and real data sets, that the predictive variance associated with the GP predictions is an effective indicator of the changes (either gradual drifts or sudden shifts) in the underlying generative distribution, and can be employed to inform the algorithm to ``re-learn'' the core manifold.
\subsection{Organization}
The rest of the paper is organized as follows: Section~\ref{sec:related_work} discusses the literature in the related areas. In Section~\ref{sec:preliminaries}, we formulate the NLDR problem and discuss preliminaries related to it. Section~\ref{sec:conv_proof} is dedicated to convergence proofs for the S-Isomap and S-Isomap++ algorithms. In Section~\ref{sec:methodology}, we present the GP-Isomap algorithm and discuss different aspects of it,  while in Section ~\ref{sec:theory}, we demonstrate the equivalence between the predictions of GP-Isomap and S-Isomap using theoretical results discussed later in Appendix~\ref{app:lemmas}. We
demonstrate the performance of our proposed algorithm on both synthetic and real-world data sets
in Section~\ref{sec:results_analysis} as well as analyze and discuss the results.


\section{Related Works}\label{sec:related_work}
Processing data streams efficiently using standard approaches is challenging in general, given streams require real-time processing and cannot be stored permanently. Any form of analysis, including detecting concept drift requires adequate summarization which can deal with the inherent constraints and that can approximate the characteristics of the stream well. Sampling based strategies include {\em random sampling}~\citep{vitter1985,chaudhuri1999} as well as decision-tree based approaches~\citep{domingos2000} which have been used in this context. To identify concept drift, maintaining statistical summaries on a streaming ``window'' is a typical strategy~\citep{alon1999,jagadish1998,datar2002}. However, none of these are applicable in the setting of learning a latent representation from the data, e.g., manifolds, in the presence of changes in the stream distribution.

We discuss limitations of existing incremental and streaming solutions that have been specifically developed in the context of manifold learning, specifically in the context of the Isomap algorithm in Section~\ref{sec:preliminaries}. Coupling Isomap with GP Regression (GPR) has been explored in the past~\citep[see][]{choi2004,xing2015}, though not in the context of streaming data. For instance, a Mercer kernel-based Isomap technique has been proposed by~\citet{choi2004}. Similarly~\citet{xing2015} presented an emulator pipeline using Isomap to determine a low-dimensional representation, whose output is fed to a GPR model.
~\citet{feragen2015} provided some theoretic results for geodesic distance kernels in their recent work and~\citet{chapelle1999} demonstrated the usage of heavy-tailed RBF kernels for the image classification task. The intuition to use GPR for detecting concept drift is novel even though the Bayesian non-parametric approach~\citep{barkan2016}, primarily intended for anomaly detection, comes close to our work in a single manifold setting. However, their choice of the Euclidean distance (in original $\mathbb{R}^D$ space) based kernel for its covariance matrix, can result in high Procrustes error, as shown in Figure~\ref{fig:pe_comparison}. Additionally, their approach does not scale, given it does not use any approximation to be able to process the new streaming points ``cheaply''.

We also note that a family of GP based non-spectral\footnote{An equivalence between GPLVM and Kernel Principal Component Analysis (KPCA) has been shown in the literature~\cite{li2016}.} non-linear dimensionality reduction methods exist, called {\em Gaussian Process Latent Variable Model} (GPLVM)~\cite{lawrence2003} and its variants~\cite{titsias2010,li2016}. GPLVM assumes that the high-dimensional observations are generated from the corresponding low-dimensional representations, using a GP prior. The latent low-dimensional representations are then inferred by maximizing the marginalized log-likelihood of the observed data, which is an optimization problem with $n$ unknown $d$-dimensional vectors, where $d$ is the length of the low-dimensional representation. In contrast, the GP-Isomap algorithm assumes that the low-dimensional representations are generated from the corresponding high-dimensional data, using a manifold-specific kernel matrix. \section{Problem Statement and Preliminaries}\label{sec:preliminaries}
We first formulate the NLDR problem and provide background on Isomap and discuss its out-of-sample and streaming extensions~\citep{bengio2004,schoeneman2017,mahapatra2017,law2006}. Additionally, we provide brief introduction to Gaussian Process (GP) analysis.
\subsection{Non-linear Dimensionality Reduction}\label{ssec:NLDR}
Given high-dimensional data ${\bf Y} = \{ {\bf y}_i \}_{i = 1 \ldots n}$, where ${\bf y}_i \in \mathbb{R}^D$, the NLDR problem is concerned with finding its corresponding low-dimensional representation ${\bf X} = \{ {\bf x}_i \}_{i = 1 \ldots n}$, such that ${\bf x}_i \in \mathbb{R}^d$, where ${d} \ll {D}$.

NLDR methods assume that the data lies along a low-dimensional manifold embedded in a high-dimensional space, and exploit the {\em global} (Isomap~\cite{tenenbaum2000}, Minimum Volume Embedding~\cite{weinberger2005}) or {\em local} (LLE~\cite{roweis2000}, Laplacian Eigenmaps~\cite{belkin2002}) properties of the manifold to map each ${\bf y}_i$ to its corresponding ${\bf x}_i$.

The Isomap algorithm~\citep{tenenbaum2000} maps each ${\bf y}_i$ to its low-dimensional representation ${\bf x}_i$ in such a way that the geodesic distance along the manifold between any two points, ${\bf y}_i$ and ${\bf y}_j$, is as close to the Euclidean distance between ${\bf x}_i$ and ${\bf x}_j$ as possible. The geodesic distance is approximated by computing the shortest path between the two points using the $k$-nearest neighbor graph\footnote{Actually, there are two variants of Isomap. The former employs a ${\bf K}$-rule to define the neighborhood $\mathcal{N}({\bf y})$ for each point ${\bf y} \in {\bf Y}$ i.e. it considers the $k$-nearest neighbors of each point ${\bf y}$ to be its neighborhood $\mathcal{N}({\bf y})$. The second variant employs a ${\bm \epsilon}$-rule to define the neighborhood $\mathcal{N}({\bf y})$ of ${\bf y}$ i.e. it considers all points which are within a radius of ${\bm \epsilon}$ to be in its neighborhood $\mathcal{N}({\bf y})$. We observe that there is a direct one-to-one relationship between the two rules with regards to computing the neighborhood $\mathcal{N}({\bf y})$ for all ${\bf y} \in {\bf Y}$.} and is stored in the geodesic distance matrix ${\bf G}= \{ {\bf g}_{i,j} \}_{1 \leq i, j \leq n}$, where ${\bf g}_{i,j}$ is the geodesic distance between the points ${\bf y}_i$ and ${\bf y}_j$. $\widetilde{{\bf G}}= \{ {{\bf g}^{2}_{i,j}} \}_{1 \leq i, j \leq n}$ contains squared geodesic distance values. The Isomap algorithm recovers ${\bf x}_i$ by using the classical {\em Multi Dimensional Scaling} (MDS) on $\widetilde{{\bf G}}$. Let ${\bf B}$ be the inner product matrix between different ${\bf x}_i$. ${\bf B}$ can be retrieved as ${\bf B} = -{\bf H}\widetilde{{\bf G}}{\bf H}/2$ by assuming $\sum\limits_{i=1}^{n} {\bf x}_i = 0$, where ${\bf H} = \{ {\bf h}_{i,j} \}_{1 \leq i, j \leq n}$ and ${\bf h}_{i,j} = {\bm \delta}_{i,j} - 1/{n}$, where ${\bm \delta}_{i,j}$ is the Kronecker delta. Isomap uncovers ${\bf X}$ such that ${\bf X}^T{\bf X}$ is as close to ${\bf B}$ as possible. This is achieved by setting ${\bf X} = \{ \sqrt{\bm \lambda}_1{\bf q}_1 \; \sqrt{\bm \lambda}_2{\bf q}_2 \; \ldots \; \sqrt{\bm \lambda}_d{\bf q}_d \}^T$ where ${\bm \lambda}_1, {\bm \lambda}_2 \dots {\bm \lambda}_d$ are the ${d}$ largest eigenvalues of ${\bf B}$ and ${\bf q}_1, {\bf q}_2 \dots {\bf q}_d$ are the corresponding eigenvectors.

The Isomap algorithm makes use of $\widetilde{{\bf G}}$ to approximate the pairwise Euclidean distances on the generated manifold. Isomap demonstrates good performance when the computed geodesic distances are close to Euclidean. In this scenario, the matrix ${\bf B}$ behaves like a positive semi-definite (PSD) kernel. The opposite scenario requires a modification to be made to $\widetilde{{\bf G}}$ to make it PSD. In MDS literature, this is commonly referred to as the Additive Constant Problem (ACP)~\citep{torgerson1952,cailliez1983, lingoes1971,cooper1972,choi2004}.

To measure error between the true, underlying low-dimensional representation to that uncovered by NLDR methods, {\em Procrustes analysis}~\citep{dryden2014} is typically used. Procrustes analysis involves aligning two matrices, ${\bf A}$ and ${\bf B}$, by finding the optimal translation ${\bf t}$, rotation ${\bf R}$, and scaling ${\bf s}$ that minimizes the {\em Frobenius norm} between the two aligned matrices, i.e.,:
\begin{equation}
{\bm \epsilon}\mbox{\textsubscript{Proc}}({\bf A}, {\bf B}) = \min_{{\bf R},{\bf t},{\bf s}} \Vert {\bf s}{\bf R}{\bf B} + {\bf t} - {\bf A}\Vert\mbox{\textsubscript{F}}
\end{equation}
The above optimization problem has a closed-form solution obtained by performing {\em Singular Value Decomposition} (SVD) of ${\bf A}{\bf B}^T$~\citep{dryden2014}. Consequently, one of the properties of Procrustes analysis is that ${\bm \epsilon}\mbox{\textsubscript{Proc}}({\bf A}, {\bf B}) = 0$ when ${\bf A} = {\bf s}{\bf R}{\bf B} + {\bf t}$ i.e. when one of the matrices is a scaled, translated and/or rotated version of the other, which we leverage upon in this work.
\subsection{Streaming Isomap}\label{ssec:sisomap}
Given that the Isomap algorithm has a complexity of $\mathcal{O}(n^3)$ (where $n$ = size of data) since it needs to perform Eigen Decomposition on ${\bf B}$ as described in the previous section, recomputing the manifold is computationally impractical to use in a streaming setting. Incremental techniques have been proposed in the past~\citep{law2006,schoeneman2017}, which can efficiently process the new streaming points, without affecting the quality of the embedding significantly. 

The S-Isomap algorithm relies on the assumption that a stable manifold can be learnt using only a fraction of the stream (denoted as the batch data set ${\bf\mathcal{B}}$), and the remaining part of stream (denoted as the stream data set ${\bf\mathcal{S}}$) can be mapped to the manifold in a significantly less costly manner. A convergence proof that justifies this assumption is provided in Section~\ref{sec:conv_proof}. Alternatively, this can be justified by considering the convergence of eigenvectors and eigenvalues of ${\bf B}$, as the number of points in the batch increase~\citep{shawe2003}. In particular, the bounds on the convergence error for a similar NLDR method, i.e., kernel PCA, is shown to be inversely proportional to the batch size~\citep{shawe2003}. Similar arguments can be made for Isomap, by considering the equivalence between Isomap and Kernel PCA~\citep{ham2004,bengio2004}. This relationship has also been empirically shown for multiple data sets~\citep{schoeneman2017}. 

The S-Isomap algorithm computes the low-dimensional representation for each new point i.e. ${{\bf x}_{n+1}}\in\mathbb{R}^d$ by solving a least-squares problem formulated by matching the dot product of the new point with the low-dimensional embedding of the points in the batch data set ${\bf X}$, computed using Isomap, to the normalized squared geodesic distances vector ${\bf f}$. The least-squares problem has the following form:
\begin{equation}
{\bf X^T}{{\bf x}_{n+1}} = {\bf f}
\end{equation}
where\footnote{Note that the Incremental Isomap algorithm~\citep{law2006} has a slightly different formulation where 
\begin{equation}
{{\bf f}_i} \simeq \frac{1}{2} \big(\frac{1}{n}\sum\limits_{j}{{\bf g}_{i,j}^2} - \frac{1}{{n}^2}\sum\limits_{l,m}{{\bf g}_{l,m}^2} \big) + \frac{1}{2}\big(\frac{1}{n}\sum\limits_{j}{{\bf g}_{j,n+1}^2} - {{\bf g}_{i,n+1}^2}\big)
\label{eqn:si2}
\end{equation}
The S-Isomap algorithm assumes that the data stream draws from an uniformly sampled, unimodal distribution $p({\bf x})$ and that the stream ${\bf \mathcal{S}}$ and the batch ${\bf \mathcal{B}}$ data sets get generated from $p({\bf x})$. Additionally it assumes that the manifold has stabilized i.e. $\left\vert{\bf \mathcal{B}}\right\vert = n$ is large enough. Using these assumptions in (\ref{eqn:si2}) above, we have that $\big(\frac{1}{n}\sum\limits_{j}{{\bf g}_{j,n+1}^2} - \frac{1}{{n}^2}\sum\limits_{l,m}{{\bf g}_{l,m}^2}\big) = {\bm \epsilon} \simeq 0$ i.e. the expectation of squared geodesic distances for points in the batch data set ${\bf \mathcal{B}}$ is close to those for points in the stream data set ${\bf \mathcal{S}}$. The line of reasoning for this follows from~\citet{hoeffding1994}. Thus~\eqref{eqn:si2} simplifies to~\eqref{eqn:si1}.} 
\begin{equation}
{{\bf f}_i} \simeq \frac{1}{2} \big(\frac{1}{n}\sum\limits_{j}{{\bf g}_{i,j}^2} - {{\bf g}_{i,n+1}^2}\big)
\label{eqn:si1}
\end{equation}

\subsection{Handling Multiple Manifolds}\label{ssec:multi}
In the ideal case, when manifolds are densely sampled and sufficiently separated, clustering can be performed before applying NLDR techniques~\citep{polito2002,fan2012}, by choosing an appropriate local neighborhood size so as not to include points from other manifolds and still be able to capture the local geometry of the manifold. However, if the manifolds are close or intersecting, such methods typically fail. While methods such as {\em Generalized Principal Component Analysis} (GPCA)~\cite{vidal2005} have been proposed to generalize linear methods such as PCA for a case where the data lies on multiple sub-spaces, such ideas have not been explored for non-linear methods.

The S-Isomap++~\citep{mahapatra2017} algorithm overcomes limitations of the S-Isomap algorithm and extends it to be able to deal with multiple manifolds. It uses the notion of Multi-scale SVD~\citep{little2009} to define tangent manifold planes at each data point, computed at the appropriate scale, and computes similarity in a local neighborhood. Additionally, it includes a novel manifold tangent clustering algorithm to be able to deal with the above issue of clustering manifolds which are close and in certain scenarios, intersecting, using these tangent manifold planes. After initially clustering the high-dimensional batch data set, the algorithm applies NLDR on each manifold individually and eventually ``stitches'' them together in a global ambient space by defining transformations which can map points from the individual low-dimensional manifolds to the global space. S-Isomap++ does not assume that the number of manifolds ($p$) is specified and automatically infers $p$ using its clustering mechanism\footnote{In cases of uneven/low density sampling, the clustering strategy discussed might possibly generate many small clusters. In such cases, one can try to merge clusters~\citep{comaniciu2002}, based on their affinity/closeness to make the clusters' size reasonable.}. Given that the data points lie on low-dimensional and potentially intersecting manifolds, it is evident that the standard clustering methods, such as K-Means~\citep{jain1999}, that operate on the observed data in $\mathbb{R}^{\bf D}$, will fail in correctly identifying the clusters.

However, S-Isomap++ can only detect manifolds which it encounters in its batch learning phase and not those which it might encounter in the streaming phase. Thus, S-Isomap++ ceases to ``learn'' and evolve to be able to limit the embedding error for points in the data stream, even though it has a ``stitching'' mechanism to embed individual low-dimensional manifolds, which might themselves be of different dimensions.
\subsection{Gaussian Process Regression}\label{ssec:gpr}
Let us assume that we are learning a probabilistic regression model to obtain the prediction at a given test input, ${\bf y}$, using a non-linear and latent function, $f(\boldsymbol{\cdot})$. Assuming\footnote{For vector-valued outputs, i.e., ${\bf x} \in \mathbb{R}^d$, one can consider $d$ independent models.} $d=1$, the observed output, $x$, is related to the input as:
\begin{equation}
x = f({\bf y}) + {\bm \varepsilon},\text{ where, } {\bm \varepsilon}\sim \mathcal{N}(0,{\bm \sigma}_n^2)
\label{eqn:0}
\end{equation}
Given a training set of inputs, ${\bf Y} = \{ {\bf y}_i \}_{i = 1 \ldots n}$ and corresponding outputs, ${\bf X} = \{x_i\}_{i = 1 \ldots n}$\footnote{While the typical notation for GPR models uses ${\bf X}$ as inputs and ${\bf Y}$ as outputs~\cite{williams2001}, we have reversed the notation to maintain consistency with rest of the paper.}, the Gaussian Process Regression (GPR) model assumes a GP prior on the latent function values, i.e., $f({\bf y}) \sim GP(m({\bf y}),k({\bf y},{\bf y}'))$, where $m({\bf y})$ is the mean of $f({\bf y})$ and $k({\bf y},{\bf y}')$ is the covariance between any two evaluations of $f(\boldsymbol{\cdot})$, i.e, $m({\bf y}) = \mathbb{E}[f({\bf y})]$ and $k({\bf y},{\bf y}') = \mathbb{E}[(f({\bf y}) - m({\bf y}))(f({\bf y}') - m({\bf y}'))]$. Here we use a zero-mean function ($m({\bf y}) = 0$), though other functions could be used as well. The GP prior states that any finite collection of the latent function evaluations are jointly Gaussian, i.e.,
\begin{equation}
f({\bf y}_1,{\bf y}_2,\ldots, {\bf y}_n) \sim \mathcal{N}({\bf 0}, K)
\label{eqn:0a}
\end{equation}
where the $ij^{th}$ entry of the $n \times n$ covariance matrix, $K$, is given by $k({\bf y}_i, {\bf y}_j)$. The GPR model uses~\eqref{eqn:0} and~\eqref{eqn:0a} to obtain the predictive distribution at a new test input, ${\bf y}_{n+1}$, as a Gaussian distribution with following mean and variance:

\begin{eqnarray}
\mathbb{E}[{x}_{n+1}] & = & {\bf k}_{n+1}^\top(K + {\bm \sigma}_n^2I)^{-1}{\bf X}\label{eqn:0c}\\
var[{x}_{n+1}] & = & k({\bf y}_{n+1},{\bf y}_{n+1}) - {\bf k}_{n+1}^\top(K + {\bm \sigma}_n^2I)^{-1}{\bf k}_{n+1} + {\bm \sigma}_n^2\label{eqn:0d}
\end{eqnarray}
where ${\bf k}_{n+1}$ is a $n \times 1$ vector with $i^{th}$ value as $k({\bf y}_{n+1},{\bf y}_i)$.

The kernel function, $k(\boldsymbol{\cdot})$, specifies the covariance between function values, $f({\bf y}_i)$ and $f({\bf y}_j)$, as a function of the corresponding inputs, ${\bf y}_i$ and ${\bf y}_j$. A popular choice is the {\em squared exponential} kernel, which has been used in this work:
\begin{equation}
k({\bf y}_i, {\bf y}_j) = {\bm \sigma}^2_s\exp{\left[-\frac{{\Vert{\bf y}_i-{\bf y}_j\Vert}^2}{2{\bm \ell}^2}\right]}
\label{eqn:gpr2}
\end{equation}
where ${\bm \sigma}_s^2$ is the signal variance and ${\bm \ell}$ is the length scale. The quantities ${\bm \sigma}_s^2$, ${\bm \ell}$, and ${\bm \sigma}_n^2$ (from Equation~\ref{eqn:0}) are the hyper-parameters of the model and can be estimated by maximizing the marginal log-likelihood of the observed data (${\bf Y}$ and ${\bf X}$) under the GP prior assumption.

One can observe that predictive mean, $\mathbb{E}[{\bf x}_{n+1}]$ in~\eqref{eqn:0c} can be written as an inner product, i.e.,:

\begin{equation}
\mathbb{E}[x_{n+1}] = {\bm \beta}^\top{\bf k}_{n+1}
\label{eqn:gpr7}
\end{equation}
where ${\bm \beta} = (K + {\bm \sigma}_n^2I)^{-1}{\bf X}$. We will utilize this form in subsequent proofs.

\section{Convergence Proofs for S-Isomap and S-Isomap++}\label{sec:conv_proof}

In this section, we demonstrate the convergence of the S-Isomap algorithm for a single manifold setting, subsequent to which we extend it to the multi-manifold setting i.e. for the S-Isomap++ algorithm described above.
\begin{thm}\label{thm:siso_result}
Given a uniformly sampled, uni-modal distribution from which the random batch data set $\mathcal{B} = \{ {\bf y}_i \in \mathbb{R}^{\bf D} \}_{i = 1 \ldots n}$ of the S-Isomap algorithm is derived from, there exists a threshold ${\bf n}_{0}$, such that when ${\bf n} \geq {\bf n}_{0}$, the Procrustes Error ${\bm \epsilon}$\textsubscript{Proc}$\big({\bm \tau}_{\mathcal{B}}$, ${\bm \tau}$\textsubscript{ISO}$\big)$ between ${\bm \tau}_{\mathcal{B}} = {\bm \phi}^{-1}\big(\mathcal{B}\big)$, the true underlying representation and ${\bm \tau}$\textsubscript{ISO}$= \hat{\bm \phi}^{-1}\big(\mathcal{B}\big)$, the embedding uncovered by Isomap is small (${\bm \epsilon}$\textsubscript{Proc} $\approx 0$) i.e. the batch phase of the S-Isomap algorithm converges, where ${\bm \phi}(\boldsymbol{\cdot})$ is the non-linear function which maps data points from the underlying low-dimensional ground truth representation ${\bf U}$ to $\mathcal{B}\in\mathbb{R}^{\bf D}$ and  the ground truth ${\bf U}$ originally resides in a convex $\mathbb{R}^{\bf d}$ Euclidean space.
\end{thm}
\begin{proof}
Based on the setting described above, the S-Isomap algorithm acts like a {\em generative model} which is trying to learn the inverse mapping ${\bm \phi}(\boldsymbol{\cdot})^{-1}$, where the associated embedding error is the Procrustes Error ${\bm \epsilon}$\textsubscript{Proc}$\big({\bm \tau}_\mathcal{B}$, ${\bm \tau}$\textsubscript{ISO}$\big)$. 


The proof follows from \cite{bernstein2000} who showed that in a setting, where given ${\bm \lambda}_1$, ${\bm \lambda}_2$, ${\bm \mu} > 0$ and for appropriately chosen ${\bm \epsilon} > 0$, as well as a data set ${\bf Y} = \{ {\bf y}_i \}_{i = 1 \ldots n}$ sampled from a Poisson distribution with density function ${\bm \alpha}$ which satisfies the ${\bm \delta}$-sampling condition i.e. 
\begin{equation}
{\bm \alpha} > \log({\bf V}/({\bm \mu} \widetilde{\bf V}({\bm \delta}/4)))/\widetilde{\bf V}({\bm \delta}/2)
\label{eqn:s_isomap_p1}
\end{equation} wherein the ${\bm \epsilon}$-rule is used to construct a graph ${\bf G}$ on ${\bf Y}$, the ratio between the graph based distance ${\bf d}_{G}({{\bf x}, {\bf y}})$ and the true Euclidean distance ${\bf d}_{M}({{\bf x}, {\bf y}}) \; \forall {\bf x}$, ${\bf y} \in {\bf Y}$ is bounded. More concretely, the following holds with probability at least $(1 - {\bm \mu})$ for $\forall {\bf x}$, ${\bf y} \in {\bf Y}$:
\begin{equation}
1 - {\bm \lambda}_1 \leq \frac{{\bf d}_{G}({{\bf x}, {\bf y}})}{{\bf d}_{M}({{\bf x}, {\bf y}})} \leq  1 + {\bm \lambda}_2
\label{eqn:s_isomap_p2}
\end{equation} where ${\bf V}$ is the volume of the manifold $\mathcal{M}$ and 
\begin{equation}
\widetilde{{\bf V}}({\bf r}) = \min\limits_{{\bf x}\in \mathcal{M}}\mbox{Vol}(\mathcal{B}_{\bf x}({\bf r})) = {\bm \eta}_{\bf d}{\bf r}^{\bf d}
\label{eqn:s_isomap_p3}
\end{equation} is the volume of the smallest metric ball in $\mathcal{M}$ of radius ${\bf r}$ and ${\bm \delta} >0$ is such that 
\begin{equation}
{\bm \delta} = {\bm \lambda}_2{\bm \epsilon}/4
\label{eqn:s_isomap_p4}
\end{equation}

A similar result can be derived in the scenario where ${\bf n}$ points are sampled independently from the fixed probability distribution $p({\bf y}$; ${\bm \theta})$, in which case we have :
\begin{equation}
{\bf n}\widetilde{\bm \alpha} = {\bm \alpha}
\label{eqn:s_isomap_p5}
\end{equation} where $\widetilde{\bm \alpha}$ is the probability of selecting a sample from $p({\bf y}$; ${\bm \theta})$.

Using \eqref{eqn:s_isomap_p3}, \eqref{eqn:s_isomap_p4} and \eqref{eqn:s_isomap_p5} in \eqref{eqn:s_isomap_p1}, we have :
\begin{equation}
\begin{split}
{\bf n}\widetilde{\bm \alpha} & > \log({\bf V}/({\bm \mu} \widetilde{\bf V}({\bm \delta}/4)))/\widetilde{\bf V}({\bm \delta}/2) \\
& = \big[ \log({\bf V}/{\bm \mu}{\bm \eta}_{d}{({\bm \lambda}_{2}{\bm \epsilon}/16)}^{\bf d}) \big]/{\bm \eta}_{d}{({\bm \lambda}_{2}{\bm \epsilon}/8)}^{\bf d}
\end{split}
\end{equation}
\begin{equation}
\begin{split}
{\bf n} & > (1/\widetilde{\bm \alpha})\big[ \log({\bf V}/{\bm \mu}{\bm \eta}_{d}{({\bm \lambda}_{2}{\bm \epsilon}/16)}^{\bf d}) \big]/{\bm \eta}_{d}{({\bm \lambda}_{2}{\bm \epsilon}/8)}^{\bf d}\\
& = {\bf n}_{0}
\end{split}
\label{eqn:s_isomap_p6}
\end{equation} where ${\bf n}_{0} = (1/\widetilde{\bm \alpha})\big[ \log({\bf V}/{\bm \mu}{\bm \eta}_{d}{({\bm \lambda}_{2}{\bm \epsilon}/16)}^{\bf d}) \big]/{\bm \eta}_{d}{({\bm \lambda}_{2}{\bm \epsilon}/8)}^{\bf d}$, is the condition which ensures that \eqref{eqn:s_isomap_p2} is satisfied.

Thus we have an adequate threshold for the size of the batch data set $\mathcal{B}$ which ensures \eqref{eqn:s_isomap_p6} is satisfied for the ${\bm \epsilon}$-rule. We can derive a similar threshold for the ${\bf K}$-rule, observing that there is a direct one-to-one mapping between ${\bf K}$ and ${\bm \epsilon}$. Refer to Section \ref{ssec:NLDR} for more details.

To complete the proof, we observe that~\eqref{eqn:s_isomap_p2} implies that ${\bf d}$\textsubscript{G}$({{\bf x},{\bf y}})$, the graph based distance between points ${\bf x}$, ${\bf y}\in {\bf G}$ is a perturbed version of ${\bf d}$\textsubscript{M}$({{\bf x}, {\bf y}})$, the true Euclidean distance between points ${\bf x}$ and ${\bf y}$ in the low-dimensional $\mathbb{R}^{\bf d}$ space. Let $\widetilde{\bf D}$\textsubscript{M} and $\widetilde{\bf D}$\textsubscript{G} represent the squared distance matrix corresponding to ${\bf d}$\textsubscript{M}$({{\bf x},{\bf y}})$ and ${\bf d}$\textsubscript{G}$({{\bf x},{\bf y}})$ respectively. Thus we have $\widetilde{\bf D}$\textsubscript{G}$ = \widetilde{\bf D}$\textsubscript{M} $ + $ $\Delta\widetilde{\bf D}$\textsubscript{M} where $\Delta\widetilde{\bf D}$\textsubscript{M}$ = \{ \Delta\widetilde{\bf d}$\textsubscript{M}$({\bf i}, {\bf j})\}_{1 \leq i, j \leq n}$ and $\Delta\widetilde{\bf d}$\textsubscript{M}$({\bf i}, {\bf j})$ are bounded due to \eqref{eqn:s_isomap_p2}.

In the past~\citep{sibson1979}, the robustness of MDS to small perturbations was demonstrated as follows. Let ${\bf F}$ represent the zero-diagonal symmetric matrix which perturbs the true squared distance matrix ${\bf B}$ to ${\bf B} + {\bm \Delta}{\bf B} = {\bf B} + {\bm \epsilon}{\bf F}$. Then the Procrustes Error between the embeddings uncovered by MDS for ${\bf B}$ and for ${\bf B} + {\bm \Delta}{\bf B}$ is given by $\frac{{\bm \epsilon}^2}{4} \sum\limits_{j, k} \frac{{{{\bf e}_j^T}{\bf F}{{\bf e}_k}}^2}{{\bm \lambda}_j+{\bm \lambda}_k}$, which is very small for small entries $\{ {\bf f}_{i,j} \}_{1 \leq i, j \leq n} \in {\bf F}$, $\{{\bf e}_k ({\bm \lambda}_k)\}_{k = 1 \ldots n}$ represent the eigenvectors (eigenvalues) of ${\bf B}$ and the double summation is over pairs of $({\bf j}, {\bf k}) = 1,2,\ldots ({\bf n}-1)$ but excluding those pairs $({\bf j}, {\bf k})$ wherein both entries of which lie in the range $({\bf K}+1), ({\bf K}+2), \ldots ({\bf n}-1)$, ${\bf K} = \sum\limits_{k=1}^{n}\mathcal{I}({\bm \lambda}_k > 0)$ and $\mathcal{I}(\boldsymbol{\cdot})$ is the indicator function. We substitute ${\bm \epsilon} = 1$ and replace ${\bf B}$ with $\widetilde{\bf D}$\textsubscript{M} and ${\bm \Delta}{\bf B}$ with ${\bm \Delta}\widetilde{\bf D}$\textsubscript{M} above to complete the proof, since the entries of ${\bm \Delta}\widetilde{\bf D}$\textsubscript{M} are very small i.e. $\{ 0 \leq {\bm \Delta}{\bf d}$\textsubscript{M}$(i, j) \leq {\bm \lambda}^2 \}_{1 \leq i, j \leq n}$ where ${\bm \lambda} = \max({\bm \lambda}_1, {\bm \lambda}_2)$ for small ${\bm \lambda}_1$, ${\bm \lambda}_2$, given the condition ${\bf n} > {\bf n}_{0}$ is satisfied for \eqref{eqn:s_isomap_p2}. Thus we have that the embedding uncovered by S-Isomap for a batch data set $\mathcal{B}$ where $\left\vert\mathcal{B}\right\vert = {\bf n} > {\bf n}_{0}$ converges asymptotically to their true embedding upto translation, rotation and scaling factors.
\end{proof}

\subsection{Extension to the Multi-manifold Setting}\label{ssec:conv_proof}
The above proof can be extended to show the convergence of the S-Isomap++~\citep{mahapatra2017} algorithm, described in Section~\ref{ssec:multi} as follows.
\begin{cor}\label{cor:siso_pp_result}
The batch phase of the S-Isomap++ algorithm converges under appropriate conditions.
\end{cor}
\begin{proof}
Similar to the proof of the convergence for the batch phase of the S-Isomap algorithm, we consider a corresponding setting for the multi-manifold scenario now, wherein we are attempting to learn the inverse mappings ${\bm \phi}(\boldsymbol{\cdot})_{i = 1, 2, \ldots p}^{-1}$ for each of the ${\bf p}$ manifolds. The initial clustering step of the S-Isomap++ algorithm separates the samples from the batch data set $\mathcal{B}$ into different individual clusters $\mathcal{B}_{i}$, such that each cluster is mutually exclusive of the others and corresponds to one of the multiple manifolds present in the data i.e. $\bigcup\limits_{i=1}^{\bf p} \mathcal{B}_{i} = \mathcal{B}$ and $\mathcal{B}_{i}\bigcap\limits_{\mathclap{\substack{\forall i,j \\ i \neq j}}}\mathcal{B}_{j} = \phi$.

The intuition for clustering and subsequently processing each of the clusters separately is based on the setting described above that the observed data was generated by first sampling points from multiple ${\bf U}_{i = 1, 2, \ldots p}$ i.e. convex domains in $\mathbb{R}^{\bf d}$ Euclidean space\footnote{It is possible that the low-dimensional Euclidean space specific to each manifold is different i.e. ${\bf U}_{i}$ is a convex domain in $\mathbb{R}^{{\bf d}_i}$ space, where ${\bf d}_i \neq {\bf d}_j$. However we can imagine a scenario where we choose a $\mathbb{R}^{\bf d}$ global space, where ${\bf d} = \sum_{i}{{\bf d}_i}$ from which the different convex ${\bf U}_{i}$ were sampled from. Additionally note that convexity is preserved by linear projections to higher dimensional spaces thus the convex domains ${\bf U}_{i = 1, 2, \ldots p}$ remain convex in this new space.} and subsequently mapping those points nonlinearly using possibly different ${\bm \phi}(\boldsymbol{\cdot})_{i = 1, 2, \ldots p}$ to $\mathcal{B}\in\mathbb{R}^{\bf D}$ space. Thus to be able to learn the different inverse mappings effectively i.e. the different ${\bm \phi}(\boldsymbol{\cdot})_{i = 1, 2, \ldots p}^{-1}$ which the S-Isomap++ algorithm strives to achieve, there is a need to be able to cluster the data appropriately.

After the initial clustering step, a similar analysis as in Theorem~\ref{thm:siso_result} provides thresholds $\exists {\bf n}_{i = 1, 2, \ldots p}$ for each of the ${\bf p}$ clusters beyond which when $\left\vert\mathcal{B}_{i}\right\vert = {\bf n} \geq {\bf n}_{i}$, the Procrustes Error ${\bm \epsilon}$\textsubscript{Proc}$\big({\bm \tau}_{\mathcal{B}_{i}}$, ${\bm \tau}$\textsubscript{ISO$_{i}$}$\big)$ between ${\bm \tau}_{\mathcal{B}_{i}} = {\bm \phi}_{i}^{-1}\big(\mathcal{B}_{i}\big)$, the true underlying representation and ${\bm \tau}$\textsubscript{ISO$_{i}$}$= \hat{\bm \phi}_{i}^{-1}\big(\mathcal{B}_{i}\big)$, the embedding uncovered by Isomap is small (${\bm \epsilon}$\textsubscript{Proc} $\approx 0$) i.e. the batch phase of the S-Isomap++ algorithm converges provided each of the ${\bf p}$ clusters $\mathcal{B}_{i = 1, 2, \ldots p}$ exceeds the appropriate threshold ${\bf n}_{i_{0}}$ (similar to Equation~\eqref{eqn:s_isomap_p6} above).
\end{proof}

The S-Isomap++ algorithm does not assume that the number of manifolds (${\bf p}$) is specified. Refer to Section \ref{ssec:multi} for more details.

\subsection{Theoretical Bounds on the Size of Batch Data Set}

\begin{figure}[!htbp]
   \centering
   \includegraphics[scale=1.3]{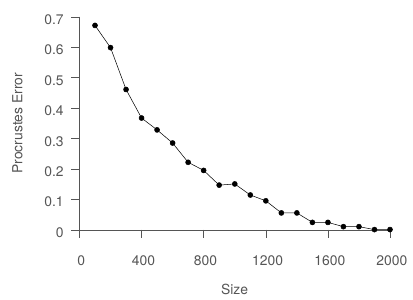}
   \caption{S-Isomap run on data samples of various size from the Euler Isometric Swiss Roll. The learned manifold is compared with the ground truth data using Procrustes error. The Procrustes error variance $\rightarrow 0$ asymptotically.}
   \label{fig:s_isomap_isometric}
\end{figure}

The threshold for the size for the batch data set $\mathcal{B}$ i.e. $\left\vert\mathcal{B}\right\vert = {\bf n} > {\bf n}_0$ beyond which the Procrustes Error converges (see Figure~\ref{fig:s_isomap_isometric}) for the synthetically generated {\em Euler Isometric Swiss Roll} for a single manifold setting is given in the Section~\ref{sec:conv_proof}. Using the result, we have : 
\[
{\bf n} > {\bf n}_0 = (\frac{1}{\widetilde{\bm \alpha}})\log((\frac{1}{\bm \mu})(\frac{\bf V}{\widetilde{\bf V}({\bm \delta}/4)}))(\frac{1}{\widetilde{\bf V}({\bm \delta}/2)})
\]

To determine the theoretical threshold ${\bf n}_0$, we substitute\footnote{$\widetilde{\bm \alpha}$ was chosen as $\textbf{1.0}$ since all points from $\mathcal{B}$ are chosen in the experiment. ${\bm \mu}$ was chosen as $\textbf{1.0}$, given $\textbf{0.0} \leq {\bm \mu} \leq \textbf{1.0}$ and thus any value chosen between $\textbf{0.0}$ and $\textbf{1.0}$ is reasonable. However we note here that ${\bm \mu}$ should be ideally chosen closer to $\textbf{0.0}$. Setting ${\bm \mu} \approx \textbf{0.0}$ gives an even higher theoretical threshold on ${\bf n}_0$, compared to the result shown in Table \ref{table:sisomap_threshold}.} the values of parameters $\widetilde{\bm \alpha}$, the probability of selecting a sample from the fixed distribution $p({\bf y}$; ${\bm \theta})$ as $\approx \textbf{1.0}$ and ${\bm \mu}$, the probability associated with the distances ratio bound as $\approx \textbf{1.0}$ and substitute parameter ${\bm \delta}$ associated with the ${\bm \delta}$-sampling condition as \textbf{0.0903}, which is estimated empirically. The value for ${\bm \eta}_d$, the volume associated with a unit ball in $\mathbb{R}^{3}$ is given by $\approx {\bm \eta}_d = \frac{4{\bm \pi}}{3} = \textbf{4.1888}$. The value for $(\frac{1}{\widetilde{\bf V}({\bm \delta}/2)})$ is given by $\frac{1}{{\bm \eta}_d*({\bm \delta}/2)*({\bm \delta}/2)*({\bm \delta}/2)} = \textbf{2593.8}$. The ratio $(\frac{\bf V}{\widetilde{\bf V}({\bm \delta}/4)})$ which is the number of balls of radius $({\bm \delta}/4)$ needed to cover the volume of manifold ${\bf V}$ is estimated empirically as $\approx \textbf{520}$. Thus the value of the theoretically estimated threshold ${\bf n}_0$ comes to $\approx (\log(520) * 2593.8) \approx \textbf{16221}$. The empirical value of threshold ${\bf n}_0$ for a single Gaussian patch (see Figure~\ref{fig:s_isomap_isometric}) is $\approx \frac{2100}{4} = \textbf{550}$. The theoretically estimated threshold on ${\bf n}_0$ is significantly larger than the empirically observed threshold on ${\bf n}_0$ in a single manifold setting for the {\em Euler Isometric Swiss Roll} data set. The theoretical prediction on ${\bf n}_0$ overestimates the empirically observed ${\bf n}_0$ for this data set i.e. we do not require a large $\mathcal{B}$ before the associated Procrustes Error starts to converge. 
\begin{table}[h!]
\centering
\begin{tabular}{| c | c | c |} 
 \hline
 & Theoretical ${\bf n}_0$ & Empirical ${\bf n}_0$\\ 
 \hline
 {\em Swiss Roll} & \textbf{16221} & \textbf{550} \\ 
 \hline
\end{tabular}
\caption{The theoretically estimated threshold ${\bf n}_0$ overestimates the empirically observed threshold ${\bf n}_0$ in a single manifold setting for the {\em Euler Isometric Swiss Roll} data set.}\label{table:sisomap_threshold}
\end{table}

\section{Methodology}\label{sec:methodology}
The proposed GP-Isomap algorithm follows a two-phase strategy (similar to the S-Isomap and S-Isomap++), where exact manifolds are learnt from an initial batch $\bf\mathcal{B}$, and subsequently a computationally inexpensive mapping procedure processes the remainder of the stream. To handle multiple manifolds, the batch data $\bf\mathcal{B}$ is first clustered via manifold tangent clustering or other standard techniques. Exact Isomap is applied on each cluster. The resulting low-dimensional data for the clusters is then ``stitched'' together to obtain the low-dimensional representation of the input data. The difference from the past methods is the mapping procedure which uses GPR to obtain the predictions for the low-dimensional mapping (see Equation~\ref{eqn:0c}). At the same time, the associated predictive variance (see Equation~\ref{eqn:0d}) is used to detect changes in the underlying distribution.

The overall GP-Isomap algorithm is outlined in Algorithm~\ref{alg:gp_isomap} and takes a batch data set, $\bf\mathcal{B}$ and the streaming data, $\bf\mathcal{S}$ as inputs, along with other parameters. The processing is split into two phases: a batch learning phase (Lines 1--15) and a streaming phase (Lines 16--32), which are described later in this section.

\begin{algorithm}[tbh]
  \caption{GP-Isomap}
  \label{alg:gp_isomap}
  \begin{multicols}{2}
  \begin{algorithmic}[1]
	\Require Batch data set: $\mathcal{B}$, Streaming data set: $\mathcal{S}$; Parameters: ${\bm \epsilon}$, ${\bm k}$, ${\bm l}$, ${\bm \lambda}$, ${\bm \sigma}_{t}$, ${\bm n}_s$
	\Ensure  $\mathcal{Y}_\mathcal{S}$: low-dimensional representation for $\mathcal{S}$
    \LineComment{Batch Phase}
    \State ${\bf \mathcal{C}}_{i=1,2 \ldots p}$ $\leftarrow$ \Call{Find\_Clusters}{$\mathcal{B}$, ${\bm \epsilon}$}
    \State ${\bm \xi}_{s}$ $\leftarrow$ $\emptyset$
    
    \For{$1\leq i\leq p$}
    \State ${\bf \mathcal{LDE}}_{i}, {\bf \mathcal{G}}_{i}$ $\leftarrow$ \Call{Isomap}{${\bf \mathcal{C}}_{i}$}
    \EndFor
    
    \For{$1\leq i\leq p$}
    \State ${\bm \phi}_{i}^{\mathcal{GP}}$ $\leftarrow$ \Call{Estimate}{${\bf \mathcal{LDE}}_{i}, {\bf \mathcal{G}}_{i}$}
    \EndFor
    
    \State ${\bm \xi}_{s}$ $\leftarrow$ $\bigcup\limits_{i=1}^{p} \bigcup\limits_{j=i+1}^{p}$ $\Call{NN}{{\bf \mathcal{C}}_{i}, {\mathcal{C}}_{j}, {\bm k}}$ $\cup$ $\Call{FN}{{\bf \mathcal{C}}_{i}, {\mathcal{C}}_{j}, {\bm l}}$
    
    \State ${\bf \mathcal{GE}}_{s}$ $\leftarrow$ \Call{MDS}{${\bm \xi}_{s}$}
    
    \For{$1\leq j\leq p$}
    
    \State ${\bf \mathcal{I}}$ $\leftarrow$ ${\bm \xi}_{s} \cap {\bf \mathcal{C}}_{j}$
    \State ${\bf \mathcal{A}}$ $\leftarrow$ \bigg[\begin{tabular}{c} ${
    \bf \mathcal{LDE}}^{\bf \mathcal{I}}_{j}$\\
    ${\bm e}^{T}$\\ \end{tabular}\bigg]
  	
  	\State ${\bf \mathcal{R}}_{i}, {t}_{i}$ $\leftarrow$ ${\bf \mathcal{GE}}_{{\bf \mathcal{I}}, s} \times {\bf \mathcal{A}}^{T} {\big( {\bf \mathcal{A}} {\bf \mathcal{A}}^{T} + {\bm \lambda} I \big)}^{-1}$
    
    \EndFor
    
    \LineComment{Streaming Phase}
    
    \State ${\bf \mathcal{S}}_{u}$ $\leftarrow$ $\emptyset$
    
    \For{$s \in {\bf \mathcal{S}}$}
    \If{$\left\vert {\bf \mathcal{S}}_{u} \right\vert$ $\geq$ ${\bm n}_{s}$}
    \State ${\bf \mathcal{Y}}_{u}$ $\leftarrow$ Re-run Batch Phase with $\mathcal{B} \leftarrow \mathcal{B} \cup {\bf \mathcal{S}}_{u}$
    \EndIf
    \For{$1\leq i\leq p$}
    	\State ${\bm \mu}_{i}, {\bm \sigma}_{i}$ $\leftarrow$ $\Call{GPR}{s, {\bf \mathcal{LDE}}_{i}, {\bf \mathcal{G}}_{i}, {\bm \phi}_{i}^{GP}}$ 
    \EndFor
    
    \State j $\leftarrow$ $\argmin_i \left\vert {\bm \sigma}_{i} \right\vert$
    \If{${\bm \sigma}_{j}$ $\leq$ ${\bm \sigma}_{t}$}
    \State $y_{s}$ $\leftarrow$ ${\bf \mathcal{R}}_{j} {\bm \mu}_j + {t}_{j}$
    \State ${\bf \mathcal{Y}}_{\bf \mathcal{S}} $ $\leftarrow$ ${\bf \mathcal{Y}}_{\bf \mathcal{S}}  \cup y_{s}$
    \Else
    \State ${\bf \mathcal{S}}_{u}$ $\leftarrow$ ${\bf \mathcal{S}}_{u} \cup {\bm s}$
    \EndIf
    \EndFor
    
    \State\Return ${\bf \mathcal{Y}}_{\bf \mathcal{S}}$
    \end{algorithmic}
    \end{multicols}
\end{algorithm}

\subsection{Kernel Function}\label{ssec:kernel}
The key innovation here is to use a manifold-specific kernel matrix in the GPR method. The matrix ${\bf B}$, which is the inner product matrix between the points in the low-dimensional space (See Section~\ref{ssec:NLDR}), could be a reasonable starting point. However, as past researchers have shown~\cite{feragen2015}, typical kernels, such as squared exponential kernel, can only be generalized to a positive definite kernel on a geodesic metric space if the space is flat. Thus ${\bf B}$ will not necessarily yield a valid positive semi-definite kernel matrix. However, a result by~\citet{cailliez1983} shows that a small positive constant, ${\bm \lambda}_\mathit{max}$, can be added to ${\bf B}$ to guarantee that it will be PSD. This constant can be calculated as the largest eigenvalue of the matrix:
\[
{\bf M} =
\begin{bmatrix}
0 & 2{\bf B} \\
-{\bf I} & -4{\bf P}
\end{bmatrix}
\]
where ${\bf P} = -{\bf H}{\bf G}{\bf H}/2$. Here, ${\bf G}$ is the geodesic distance matrix and ${\bf H} = \{ {\bf h}_{i,j} \}_{1 \leq i, j \leq n}$, ${\bf h}_{i,j} = {\bm \delta}_{i,j} - 1/{n}$, where ${\bm \delta}_{i,j}$ is the Kronecker delta.

Thus using ~\citep{cailliez1983}, we construct $\widetilde{{\bf B}}$ from ${\bf B}$ as follows
\[
\widetilde{{\bf B}} = {\bf B} + 2{\bm \lambda}_\mathit{max}{\bf P} + \frac{1}{2}{\bm \lambda}_\mathit{max}^2{\bf H}
\]
where ${\bm \lambda}_\mathit{max}$ is the largest eigenvalue of ${\bf M}$.

The proposed GP-Isomap algorithm uses a novel geodesic distance based kernel function defined as:
\begin{equation}
k({\bf y}_i,{\bf y}_j) = {\bm \sigma}^2_s\exp\left(-\frac{\widetilde{{\bf b}}_{i,j}}{2{\bm \ell}^2}\right)
\label{eqn:k0}
\end{equation}
where $\widetilde{{\bf b}}_{i,j}$ is the ${ij}^{th}$ entry of the matrix $\widetilde{{\bf B}}$, ${\bm \sigma}^2_s$ is the signal variance (whose value we fix as 1 in this work) and ${\bm \ell}$ is the length scale hyper-parameter. Thus the kernel matrix ${\bf K}$ can be written as:
\begin{equation}
{\bf K} = \exp{\left(-\frac{\widetilde{{\bf B}}}{2{\bm \ell}^2}\right)}
\label{eqn:k1}
\end{equation}
This kernel function plays a key role in using the GPR model for mapping streaming points on the learnt manifold, by measuring similarity along the low-dimensional manifold, instead of the original space ($\mathbb{R}^D$), as is typically done in GPR based solutions.

The matrix $\widetilde{{\bf B}}$, is positive semi-definite. Consequently, we note that the kernel matrix, ${\bf K}$, is positive definite (refer Equation~\ref{eqn:k2} below).


Using Lemma~\ref{lemma1}, the novel kernel we propose can be written as
\begin{equation}
{\bf K}\big( {\bf x}, {\bf y} \big) = {\bf I} + \sum\limits_{i=1}^{d} \big[ \exp{\left(-\frac{{\bm \lambda}_i}{2{\bm \ell}^2}\right)} - 1 \big]{{\bf q}_i}{{\bf q}_i^T} = {\bf I} + {\bf Q}\widetilde{\bf \Lambda}{\bf Q ^{T}}
\label{eqn:k2}
\end{equation}
where $\widetilde{\bf \Lambda} = \begin{bmatrix} 
    \big[ \exp{\left(-\frac{{\bm \lambda}_1}{2{\bm \ell}^2}\right)} - 1 \big] & 0 & 0 \\
    0 & \ddots & 0 \\
    0 &   0     & \big[ \exp{\left(-\frac{{\bm \lambda}_d}{2{\bm \ell}^2}\right)} - 1 \big] 
    \end{bmatrix}$ and $\{{\bm \lambda}_i, {\bf q}_i \}_{i = 1 \ldots d}$ are eigenvalue/eigenvector pairs of $\widetilde{{\bf B}}$ as discussed in Section~\ref{ssec:NLDR}.

\subsection{Batch Learning}\label{ssec:batch}
The batch learning phase consists of these tasks : 
\subsubsection{Clustering.}\label{sssec:clustering}
The first step in the batch phase involves clustering of the batch data set $\bf\mathcal{B}$ into ${\bm p}$ individual clusters which represent the manifolds. In case, $\bf\mathcal{B}$ contains a single cluster, the algorithm can correctly detect it. Refer to Section \ref{ssec:multi} for more details. (Line 1)
\subsubsection{Dimension Reduction.}\label{sssec:dimred}
Subsequently, full Isomap is executed on each of the ${\bm p}$ individual clusters to get low-dimensional representations ${\bf \mathcal{LDE}}_{i=1,2 \ldots p}$ of the data points belonging to each individual cluster. (Lines 3--5)
\subsubsection{Hyper-parameter Estimation.}\label{sssec:hypest}The geodesic distance matrix for the points in the ${\bm i}$\textsuperscript{th} manifold  ${{\bf \mathcal{G}}_{i}}$ and the corresponding low-dimensional representation ${{\bf \mathcal{LDE}}_{i}}$, are fed to the GP model for each of the ${\bm p}$ manifolds, to perform hyper-parameter estimation, which outputs $\{{\bm \phi}_{i}^{GP} \}_{i = 1,2 \ldots p}$. (Lines 6--8)
\subsubsection{Learning Mapping to Global Space.}\label{sssec:learning}The low-dimensional embedding uncovered for each of the manifolds can be of different dimensionalities. Consequently, a mapping to a unified global space is needed. To learn this mapping, a support set ${\bm \xi}_{s}$ is formulated, which contains the ${\bm k}$ pairs of nearest points and ${\bm l}$ pairs of farthest points, between each pair of manifolds. Subsequently, MDS is executed on this support set ${\bm \xi}_{s}$ to uncover its low-dimensional representation ${\bf \mathcal{GE}}_{s}$. Individual scaling and translation factors $\{ {\bf \mathcal{R}}_{i}, {t}_{i} \}_{i = 1,2 \ldots p}$ are learnt via solving a least squares problem involving ${\bm \xi}_{s}$, which map points from each of the individual manifolds to the global space. (Lines 9--15)
\subsection{Stream Processing}\label{ssec:stream}
In the streaming phase, each sample ${\bm s}$ in the stream set $\bf \mathcal{S}$ is embedded using each of the $\bm p$ GP models to evaluate the prediction ${\bm \mu}_{i}$, along with the variance ${\bm \sigma}_{i}$ (Lines 22--24). The manifold with the smallest variance get chosen to embed the sample ${\bm s}$ into, using the corresponding scaling ${\bf \mathcal{R}}_{j}$ and translation factor ${t}_{j}$, provided ${ min_i} \left\vert {\bm \sigma}_{i} \right\vert$ is within the allowed threshold ${\bm \sigma}_{t}$ (Lines 25--28), otherwise sample ${\bm s}$ is added to the unassigned set ${\bf \mathcal{S}}_{u}$ (Lines 29--31). When the size of unassigned set ${\bf \mathcal{S}}_{u}$ exceeds certain threshold ${\bm n}_{s}$, we add them to the batch data set and re-learn the base manifold (Line 18--20). The assimilation of the new points in the batch maybe done more efficiently in an incremental manner.

\subsection{Complexity}\label{ssec:complexity}
The runtime complexity of our proposed algorithm is dominated by the GP regression step as well as the Isomap execution step, both of which have $\mathcal{O}(n^3)$ complexity, where $n$ is the size of the batch data set $\mathcal{B}$. This is similar to the S-Isomap and S-Isomap++ algorithms, that also have a runtime complexity of $\mathcal{O}(n^3)$. The stream processing step is $\mathcal{O}(n)$ for each incoming streaming point. The space complexity of GP-Isomap is dominated by $\mathcal{O}(n^2)$. This is because each of the samples of the stream set $\mathcal{S}$ get processed separately. Thus, the space requirement as well as runtime complexity does not grow with the size of the stream, which makes the algorithm appealing for handling high-volume streams.

\section{Theoretical Analysis}\label{sec:theory}
In this section, we first state the main result and subsequently prove it using results from lemmas stated later in Appendix~\ref{app:lemmas}. Mention how this will be extended to a multi-manifold case.
\begin{thm}\label{thm:result}
For a single manifold setting, the prediction ${\bm \tau}$\textsubscript{GP} of GP-Isomap is equivalent to the prediction ${\bm \tau}$\textsubscript{ISO} of S-Isomap i.e. the Procrustes Error ${\bm \epsilon} $\textsubscript{Proc}$\big({\bm \tau}$\textsubscript{GP}, ${\bm \tau}$\textsubscript{ISO}$\big)$ between ${\bm \tau}$\textsubscript{GP} and ${\bm \tau}$\textsubscript{ISO} is $0$.
\end{thm}
\begin{proof}
The prediction of GP-Isomap is given by \eqref{eqn:gpr7}. Using Lemma~\ref{lemma5}, we demonstrated that
\begin{equation}
{\bm \beta} = \{ \frac{{\bm \alpha}\sqrt{\bm \lambda}_1{\bf q}_1}{{1 + {\bm \alpha}{{\bf c}_1}}} \frac{{\bm \alpha}\sqrt{\bm \lambda}_2{\bf q}_2}{{1 + {\bm \alpha}{{\bf c}_2}}} \ldots \frac{{\bm \alpha}\sqrt{\bm \lambda}_d{\bf q}_d}{{1 + {\bm \alpha}{{\bf c}_d}}} \}
\label{eqn:thm1}
\end{equation}
The term ${\bf K_{*}}$ for GP-Isomap, using our novel kernel function evaluates to 
\begin{equation}
{\bf K_{*}} = \exp{\left(-\frac{{\bf G}_{*}^2}{2{\bm \ell}^2}\right)}
\label{eqn:thm2}
\end{equation}
where ${\bf G}_{*}^2$ represents the vector containing the squared geodesic distances of ${\bf x_{n+1}}$ to ${\bf X}$ containing $\{ {\bf x}_i \}_{i = 1,2 \ldots n}$.

Considering the above equation element-wise, we have that the ${\bf i}$\textsuperscript{th} term of ${\bf K_{*}}$ equates to $\exp{\left[-\frac{{\bf g}_{i,n+1}^2}{2{\bm \ell}^2}\right]}$. Using Taylor's series expansion we have, 
\begin{equation}
\exp{\left[-\frac{{\bf g}_{i,n+1}^2}{2{\bm \ell}^2}\right]} \simeq \big(1 -\frac{{\bf g}_{i,n+1}^2}{2{\bm \ell}^2}\big) \mbox{ for large } {\bm \ell}
\label{eqn:thm3}
\end{equation}

The prediction by the S-Isomap is given by \eqref{eqn:si1} as follows :-
\begin{equation}
{\bm \tau}\textsubscript{ISO} = \{ \sqrt{\bm \lambda}_1{\bf q}_1^T{\bf f} \; \sqrt{\bm \lambda}_2{\bf q}_2^T{\bf f} \; \ldots \; \sqrt{\bm \lambda}_d{\bf q}_d^T{\bf f} \}^T
\label{eqn:thm4}
\end{equation}
where ${\bf f} = \{ {\bf f}_i\}$ is as defined by \eqref{eqn:si1}.

Rewriting \eqref{eqn:si1} we have,
\begin{equation}
{{\bf f}_i} \simeq \frac{1}{2} \big({\bm \gamma} - {{\bf g}_{i,n+1}^2}\big)
\label{eqn:thm5}
\end{equation}
where ${\bm \gamma} = \big(\frac{1}{n}\sum\limits_{j}{{\bf g}_{i,j}^2} \big)$ is a constant with respect to ${\bf x}_{n+1}$, since it depends only on squared geodesic distance values associated within the batch data set $\bf \mathcal{B}$ and ${\bf x}_{n+1}$ is part of the stream data set $\bf \mathcal{S}$.

We now consider the ${1}$\textsuperscript{st} dimension of the predictions for GP-Isomap and S-Isomap only and demonstrate their equivalence via Procrustes Error. The analysis for the remaining dimensions follows a similar line of reasoning.

Thus for the ${1}$\textsuperscript{st} dimension, using~\eqref{eqn:thm5} the S-Isomap prediction is
\begin{equation}
\begin{split}
{\bm \tau}\textsubscript{ISO}_{1}& = \sqrt{\bm \lambda}_1{\bf q}_1^T{\bf f} \\
& = \sqrt{\bm \lambda}_1\sum\limits_{i=1}^{n} {{\bf q}_{1,i}} \big(\frac{1}{2} \big({\bm \gamma} - {{\bf g}_{i,n+1}^2}\big)\big)\\
& = \frac{\sqrt{\bm \lambda}_1}{2} \sum\limits_{i=1}^{n} {{\bf q}_{1,i}} \big({\bm \gamma} - {{\bf g}_{i,n+1}^2}\big)\\
\end{split}
\label{eqn:thm6}
\end{equation}
Similarly using Lemma~\ref{lemma5},~\eqref{eqn:thm2} and~\eqref{eqn:thm3}, we have that the ${\bf 1}$\textsuperscript{st} dimension for GP-Isomap prediction is given by, 
\begin{equation}
\begin{split}
{\bm \tau}\textsubscript{GP}_{1}& = \frac{{\bm \alpha}\sqrt{\bm \lambda}_1{\bf q}_1^T}{{1 + {\bm \alpha}{{\bf c}_1}}} {\bf K_{*}} \\
& = \frac{{\bm \alpha}\sqrt{\bm \lambda}_1}{1 + {\bm \alpha}{{\bf c}_1}} \sum\limits_{i=1}^{n} {{\bf q}_{1,i}} \big(1 - \frac{{{\bf g}_{i,n+1}^2}}{2{\bm \ell}^2}\big)\\
\end{split}
\label{eqn:thm7}
\end{equation}
We can observe that ${\bm \tau}\textsubscript{GP}_{1}$ is a scaled and translated version of ${\bm \tau}\textsubscript{ISO}_{1}$. Similarly for each of the dimensions (${1} \leq {i} \leq {d}$), the prediction for the GP-Isomap ${\bm \tau}\textsubscript{GP}_{i}$ can be shown to be a scaled and translated version of the prediction for the S-Isomap ${\bm \tau}\textsubscript{ISO}_{i}$. These individual scaling ${\bf s}_i$ and translation ${\bf t}_i$ factors can be represented together by single collective scaling ${\bf s}$ and translation ${\bf t}$ factors. Consequently, the Procrustes Error ${\bm \epsilon} $\textsubscript{Proc}$ \big({\bm \tau}$\textsubscript{GP}, ${\bm \tau}$\textsubscript{SI}$\big)$ is 0. (refer Section~\ref{ssec:NLDR}).
\end{proof}

\section{Results and Analysis}\label{sec:results_analysis}
In this section, we demonstrate the performance of the proposed algorithm on both synthetic and real-world data sets. In Section~\ref{ssec:synthetic_results}, we present results for synthetic data sets, whereas Section~\ref{ssec:sensor_results} contains results on benchmark sensor data sets. Our results demonstrate that:
\begin{inparaenum}[i).]
\item GP-Isomap is able to perform good quality dimension reduction on a manifold, 
\item the reduction produced by GP-Isomap is equivalent to the corresponding output of S-Ismap (or S-Isomap++), and
\item the predictive variance within GP-Isomap is able to identify changes in the underlying distribution in the data stream on all data sets considered in this paper.
\end{inparaenum}

GP-Isomap has the following hyper-parameters: $\epsilon$, $k$, $l$, $\lambda$, $\sigma_t$, $n_s$. We set $k$, $l$, $\lambda$ to have values of $16$, $1$ and $0.005$, respectively, based on past results for S-Isomap++~\citep{mahapatra2017}. We study the effect of $\sigma_t$ and $n_s$ using the different data sets listed in Sections~\ref{ssec:synthetic_results} and~\ref{ssec:sensor_results}.


\subsection{Results on Synthetic Data Sets}\label{ssec:synthetic_results}

Swiss roll data sets are typically used for evaluating manifold learning algorithms. To evaluate our method on concept drift, we use the Euler Isometric Swiss Roll data set~\citep{schoeneman2017} consisting of four $\mathbb{R}^{2}$ Gaussian patches having $n=2000$ points each, chosen at random, which are embedded into $\mathbb{R}^{3}$ using a non-linear function ${\bm \psi}(\cdot)$. The points for each of the Gaussian modes were divided equally into training and test sets randomly. To test incremental concept drift, we use one of the training data sets from the above data set, along with a uniform distribution of points for testing (refer to Figure~\ref{fig:expts} for details). Figures~\ref{fig:expts},~\ref{fig:s_isomap_isometric},~\ref{fig:patches_threshold},~\ref{fig:patches_pred_comp} and~\ref{fig:pe_comparison} demonstrates our results on this data set. 

\begin{figure}[!htbp]
  \centering
  \includegraphics[scale=0.4]{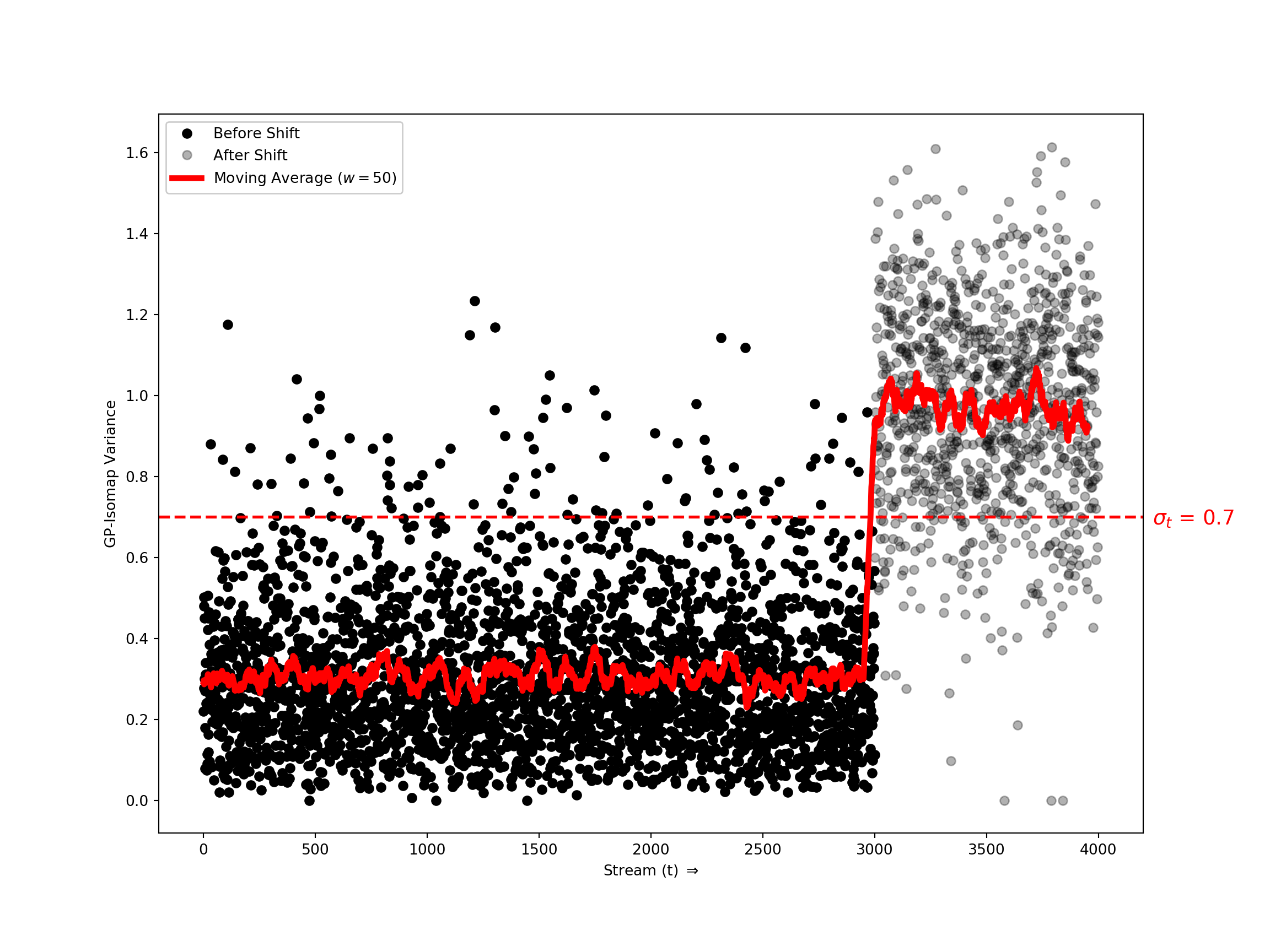}
  \caption{Using variance to detect concept drift for the Euler Isometric Swiss Roll data set. The horizontal axis represents time and the vertical axis represents variance of the stream. Initially, when stream consists of samples generated from known modes, variance is low, later when samples from an unrecognized mode appear i.e. $t \geq 3000$, variance shoots up significantly. There is some noise in the stream which results in some instances getting wrongly classified to belong to a different mode. The optimal values of hyper-parameters $n_s$ and $\sigma_t$ were determined to be $1000$ and $0.7$ respectively for this data set.}
\label{fig:patches_threshold}
\end{figure}

\begin{figure}[!htbp]
\centering
\includegraphics[scale=0.18]{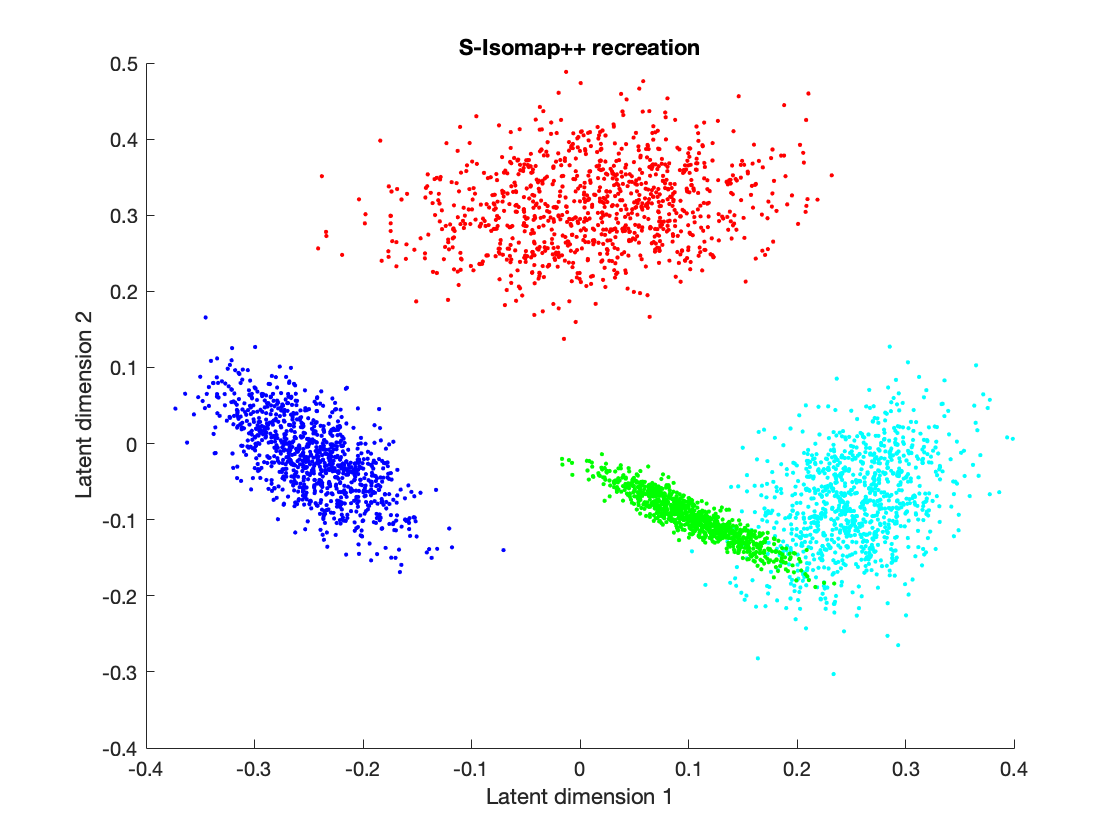}\includegraphics[scale=0.18]{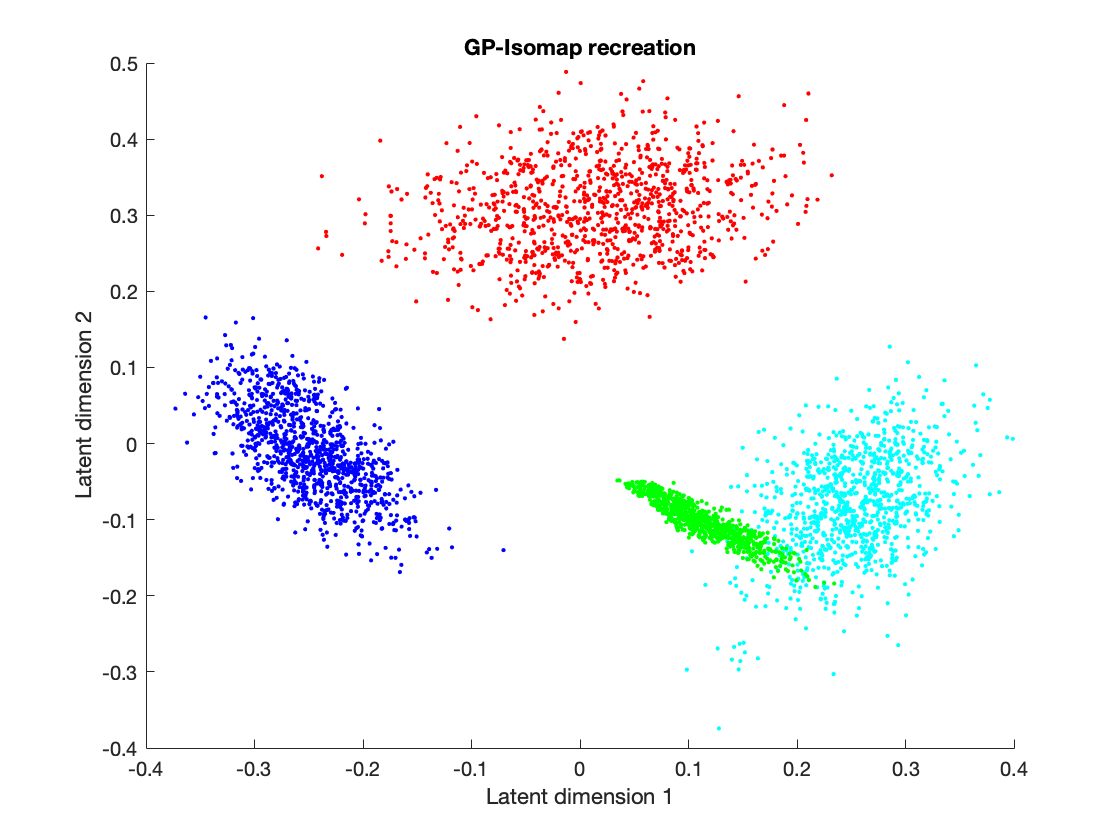}
\caption{Comparing predictions for S-Isomap++ and GP-Isomap empirically for the Euler Isometric Swiss Roll data set. The low-dimensional representations uncovered by each are almost similar.}
\label{fig:patches_pred_comp}
\end{figure}

\subsubsection{Gaussian patches on Isometric Swiss Roll}\label{sssec:gaussian_results}  To evaluate our method on sudden concept drift, we trained our GP-Isomap model using the first three out of four training sets of the Euler Isometric Swiss Roll data set. Subsequently we stream points randomly from the test sets from only the first three classes initially and later stream points from the test set of the fourth class, keeping track of the predictive variance all the while. Figure~\ref{fig:patches_threshold} demonstrates the sudden increase (see red line) in the variance of the stream when streaming points are from the fourth class i.e. unknown mode. Thus GP-Isomap is able to detect concept drift correctly. The bottom panel of Figure~\ref{fig:expts} demonstrates the performance of S-Isomap++ on this data set. It fails to map the streaming points of the unknown mode correctly, given it had not encountered the unknown mode during the batch training phase.

To test our proposed approach for detecting incremental concept drift, we train our model using the single patch data set and subsequently observe how the variance of the stream behaves on the test streaming data set. The top panel of Figure~\ref{fig:expts} shows how gradually variance increases smoothly as the stream gradually drifts away from the Gaussian patch. This shows that GP-Isomap maps incremental drift correctly. In Section~\ref{sec:theory}, we proved the equivalence between the prediction of S-Isomap with that of GP-Isomap, using our novel kernel. In Figure~\ref{fig:pe_comparison}, we show empirically via Procrustes Error (PE) that indeed the prediction of S-Isomap matches that of GP-Isomap, irrespective of size of batch used. PE for GP-Isomap with the Euclidean distance based kernel remains high irrespective of the size of the batch, which clearly demonstrates the unsuitability of this kernel to adequately learn mappings in the low-dimensional space.

\begin{figure}[!htbp]
  \centering
  \includegraphics[scale=0.27]{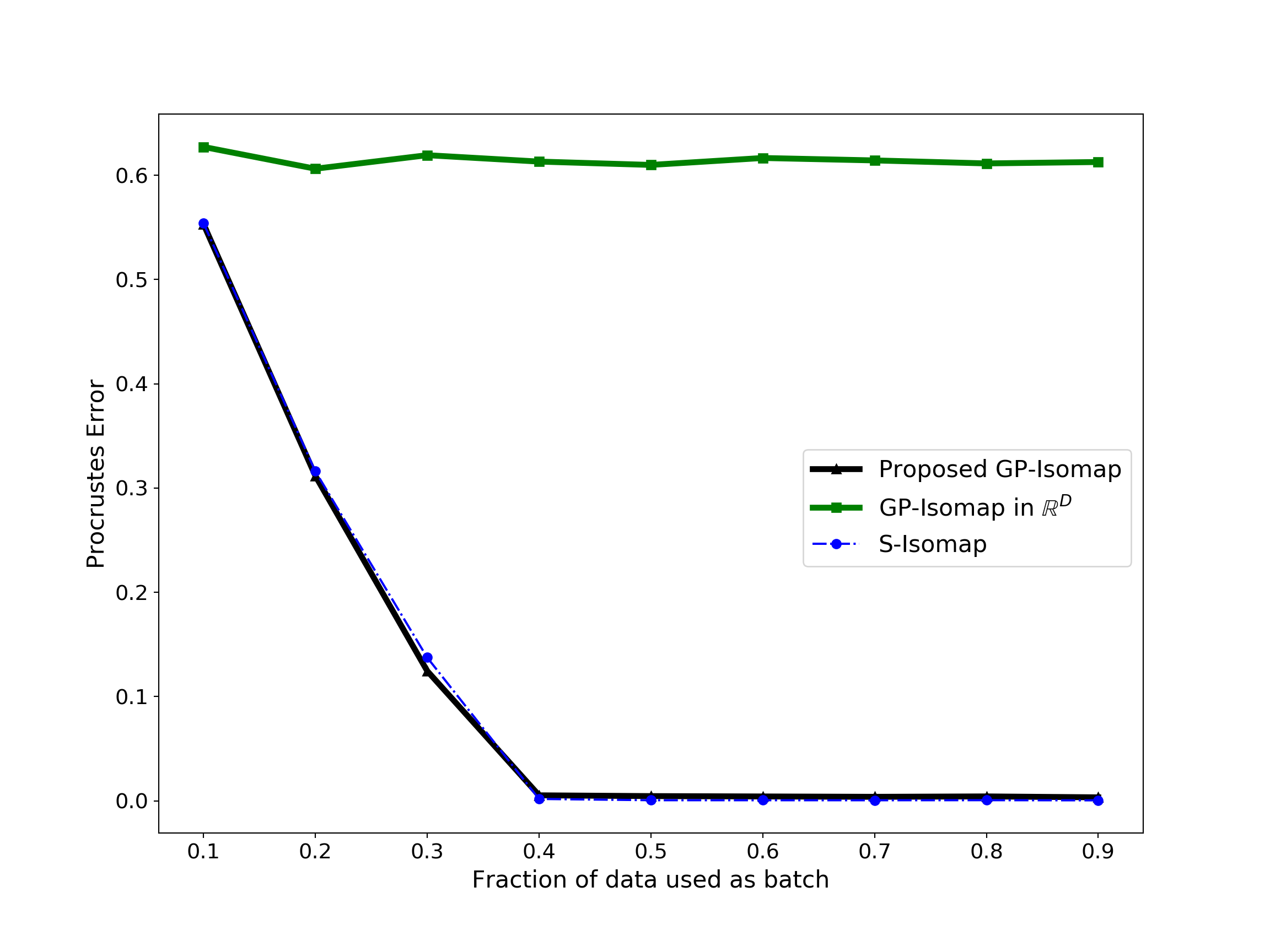}
  \caption{Procrustes error (PE) between the ground truth with a) GP-Isomap (blue line) with the geodesic distance based kernel, b) S-Isomap (dashed blue line with dots) and c) GP-Isomap (green line) using the Euclidean distance based kernel, for different fractions (${f}$) of data used in the batch $\bf\mathcal{B}$. The behavior of PE for a) closely matches that for b). However, the PE for GP-Isomap using the Euclidean distance kernel remains high irrespective of ${f}$ demonstrating its unsuitability for manifolds. }
\label{fig:pe_comparison}
\end{figure}

\subsection{Results on Sensor Data Set}\label{ssec:sensor_results}
In this section, we present results from different benchmark sensor data sets to demonstrate the efficacy of our algorithm.

\subsubsection{Results on Gas Sensor Array Drift Data Set}\label{sssec:gsad_results}

The Gas Sensor Array Drift~\citep{vergara2012} data set is a benchmark data set ($n = 13910$) available to research communities to develop strategies to dealing with concept drift and uses measurements from 16 chemical sensors used to discriminate between 6 gases (class labels) at various concentrations. We demonstrate the performance of our proposed method on this data set. 

We first removing instances which had  invalid/empty entries as feature values. 
Subsequently the data was {\em mean normalized}. Data points from the first five classes were divided into training and test sets. We train our model using the training data from four out of these five classes. While testing, we stream points randomly from the test sets of these four classes first and later stream points from the test set of the fifth class. Figures~\ref{fig:gsad_threshold} and~\ref{fig:gsad_pred_comp} demonstrate our results on this data set. From figure~\ref{fig:gsad_threshold}, we observe that our model can clearly detect concept drift due to the unknown fifth class by tracking the variance of the stream, using the running average (red line). While we have already demonstrated the equivalence between the prediction of S-Isomap with that of GP-Isomap in  Section~\ref{sec:theory}, figure~\ref{fig:gsad_pred_comp} demonstrates the equivalence empirically where we can clearly observe that the low-dimensional representations uncovered by both algorithms are similar.

\begin{figure}[!htbp]
  \centering
  \includegraphics[scale=0.4]{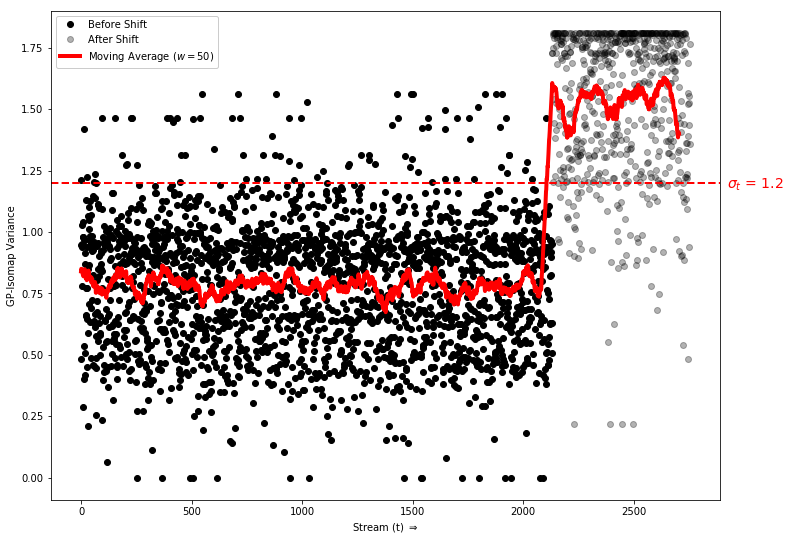}
  \caption{Using variance to identify concept drift for the Gas Sensor Array Drift data set. Similar to Figure~\ref{fig:patches_threshold}, the introduction of points from an unknown mode in the stream results in variance increasing drastically as demonstrated by the mean (red line). The spread of variances for points from known modes ($t \precsim 2000$) is also smaller, compared to the spread for the points from the unknown mode ($t \succsim 2000$). Noise results in some instances getting mis-classified. The optimal values of hyper-parameters $n_s$ and $\sigma_t$ were determined to be $412$ and $1.2$ respectively for this data set.}
\label{fig:gsad_threshold}
\end{figure}
\begin{figure}[!htbp]
    \centering
    \begin{subfigure}[t]{.49\textwidth}
        \centering
        \includegraphics[width=\textwidth]{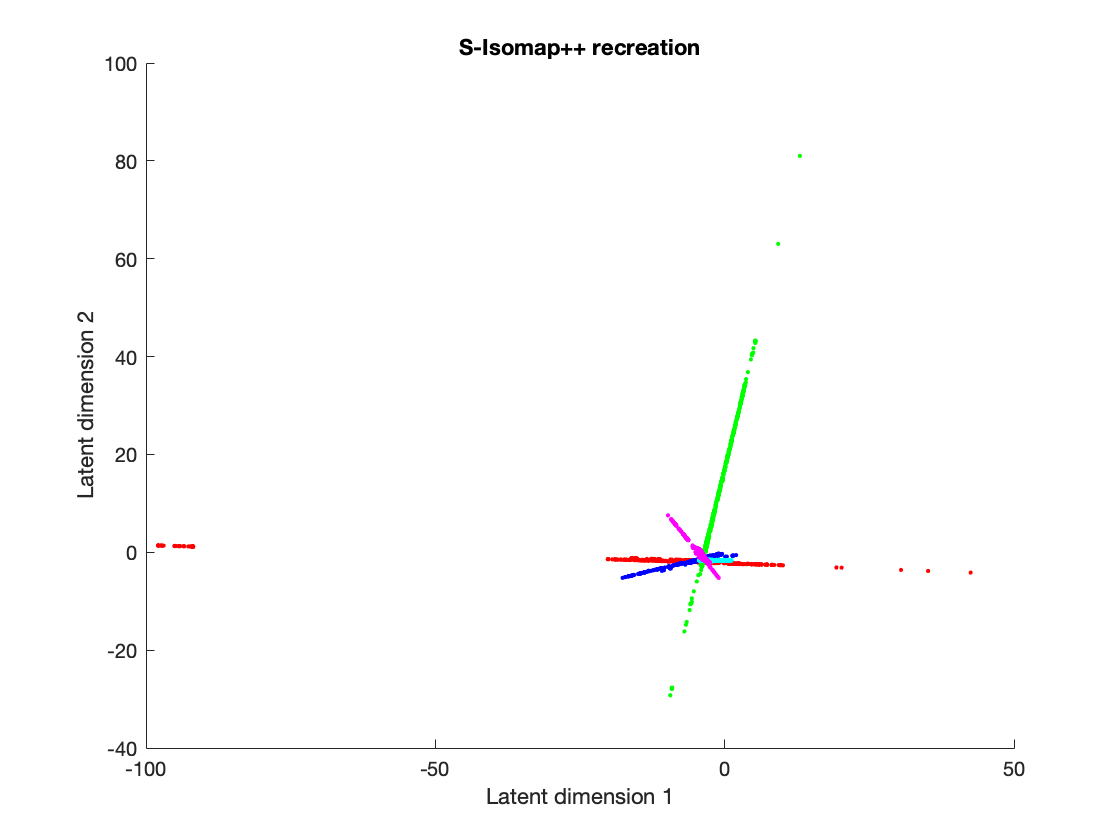}
        \caption{S-Isomap recreation}
    \end{subfigure}
    \hfill
    \begin{subfigure}[t]{.49\textwidth}
        \centering
        \includegraphics[width=\textwidth]{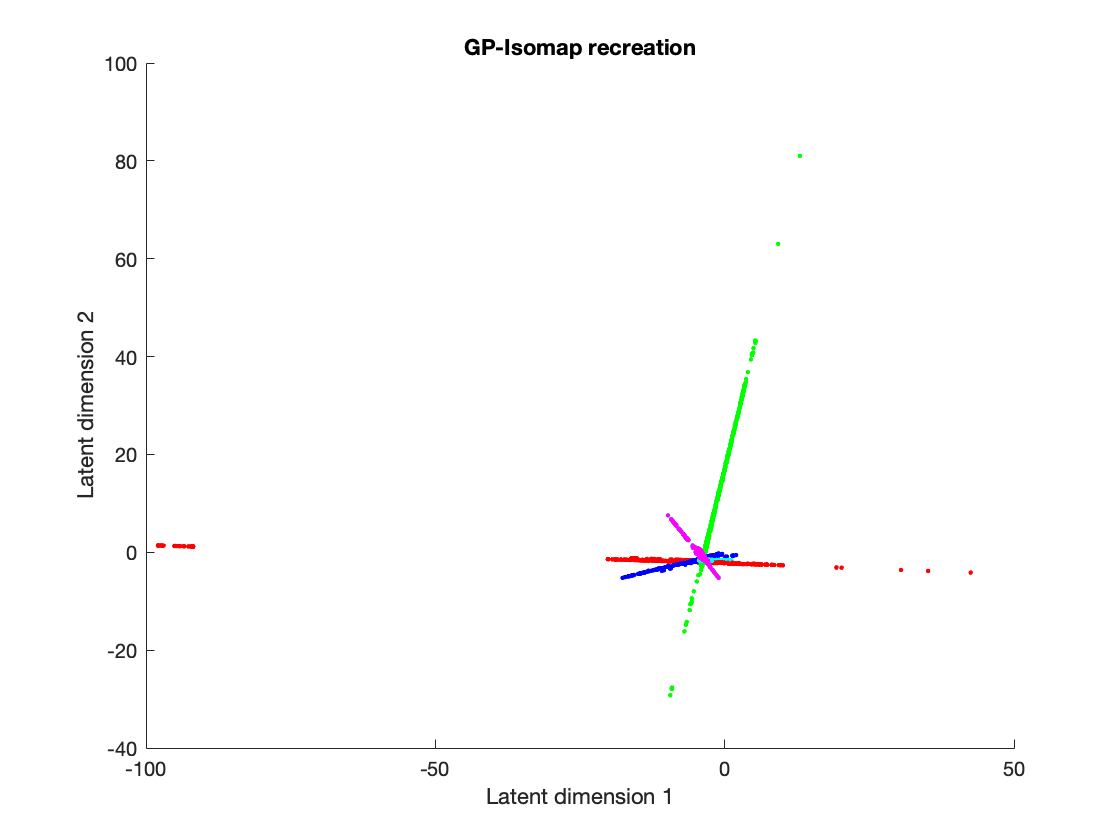}
        \caption{GP-Isomap recreation}
    \end{subfigure}
    \caption{Comparing predictions for S-Isomap++ and GP-Isomap empirically for the Gas Sensor Array Drift data set. We observe that the low-dimensional representations uncovered by both algorithms are equivalent.}
    \label{fig:gsad_pred_comp}
\end{figure}

\subsubsection{Results on Human Activity Recognition (HAR) Data Set}\label{sssec:har_results}

The Human Activity Recognition~\citep{velloso2013} data set consists of multiple data sets which are focused on discriminating between different activities, i.e. to predict which activity was performed at a specific point in time. In this work, we focused on the Weight Lifting Exercises (WLE) data set ($n = 39242$) which investigates how well an activity was performed by the wearer of different sensor devices. The WLE data set consists of six young health participants who performed one set of 10 repetitions of the Unilateral Dumbbell Biceps Curl in five different fashions: exactly according to the specification (Class A), throwing the elbows to the front (Class B), lifting the dumbbell only halfway (Class C), lowering the dumbbell only halfway (Class D) and throwing the hips to the front (Class E). Class A corresponds to the specified execution of the exercise, while the other 4 classes correspond to common mistakes.

The data set was cleaned i.e. instances with invalid/empty entries were removed. Subsequently the data points from the different classes were {\em mean normalized} and divided into training and test sets. Figures~\ref{fig:wle_threshold} and~\ref{fig:wle_pred_comp} demonstrate our results on this data set. While figure~\ref{fig:wle_threshold} demonstrates the concept drift phenomenon adequately, figure~\ref{fig:wle_pred_comp} compares the predictions for the S-Isomap++ and GP-Isomap algorithms empirically on this data set. In figure~\ref{fig:wle_threshold}, similar to the methodology we used earlier to detect concept drift, we initially trained our algorithm using instances from the latter four classes only, whereas during the streaming phase we randomly selected instances from the streaming set of these four classes first and later streamed points from the first class, keeping track of the predictive variance all the while. Figure~\ref{fig:wle_pred_comp} demonstrates the equivalence between the output of the S-Isomap++ and GP-Isomap algorithms empirically. 

\begin{figure}[!htbp]
  \centering
  \includegraphics[scale=0.4]{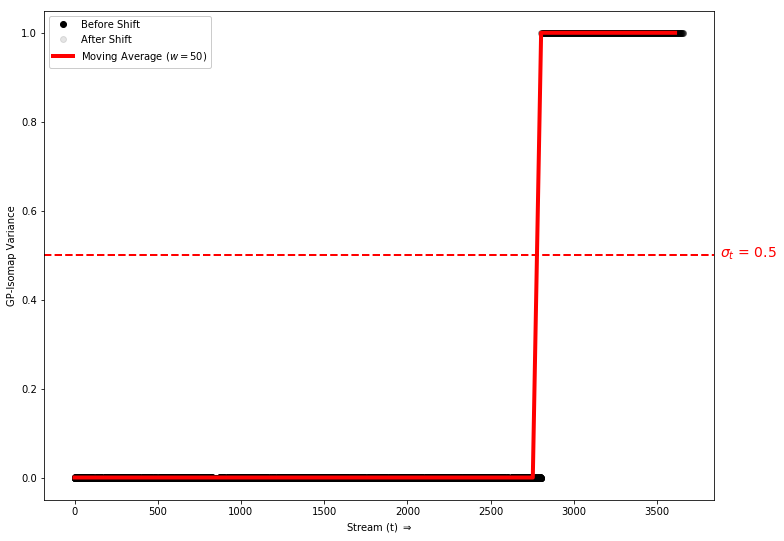}
  \caption{Using variance to detect concept drift using the Human Activity Recognition data set. The horizontal axis represents time and the vertical axis represents variance of the stream. Initially, when stream consists of samples generated from known modes, the stream variance is low, later when samples from an unrecognized mode appear i.e. $t \succsim 2500$, variance shoots up drastically as demonstrated by the mean (red line). The variance was well-behaved for this experiment. The optimal values of hyper-parameters $n_s$ and $\sigma_t$ were determined to be $855$ and $0.5$ respectively for this data set.}
\label{fig:wle_threshold}
\end{figure}
\begin{figure}[!htbp]
    \centering
    \begin{subfigure}[t]{.48\textwidth}
        \centering
        \includegraphics[scale=0.45]{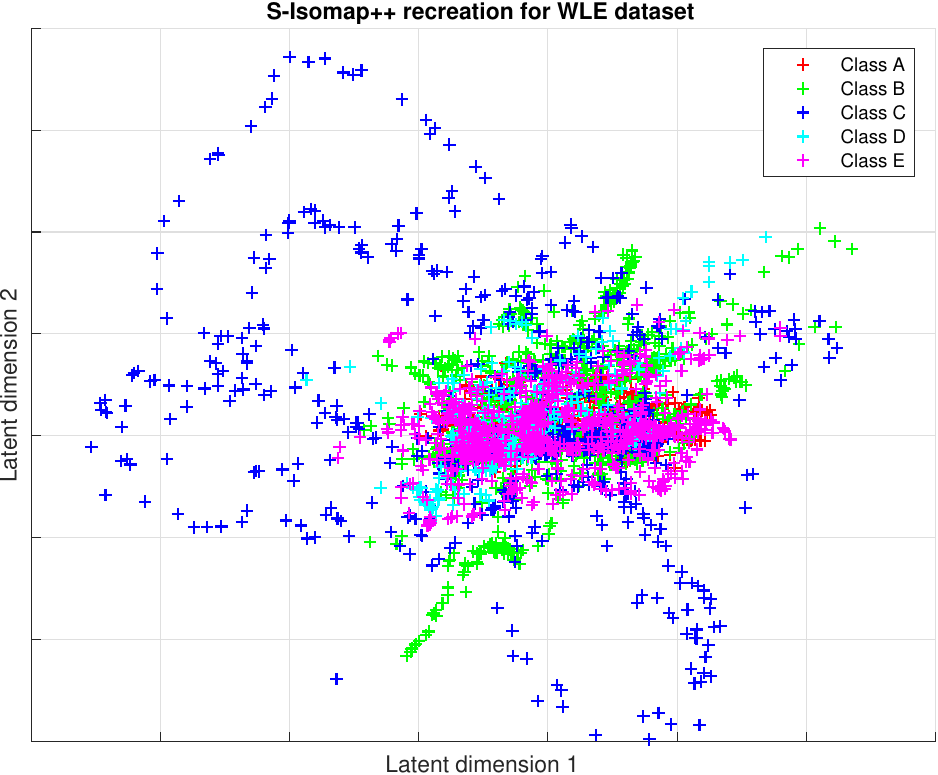}
        \caption{S-Isomap recreation}
    \end{subfigure}
    \hfill
    \begin{subfigure}[t]{.48\textwidth}
        \centering
        \includegraphics[scale=0.45]{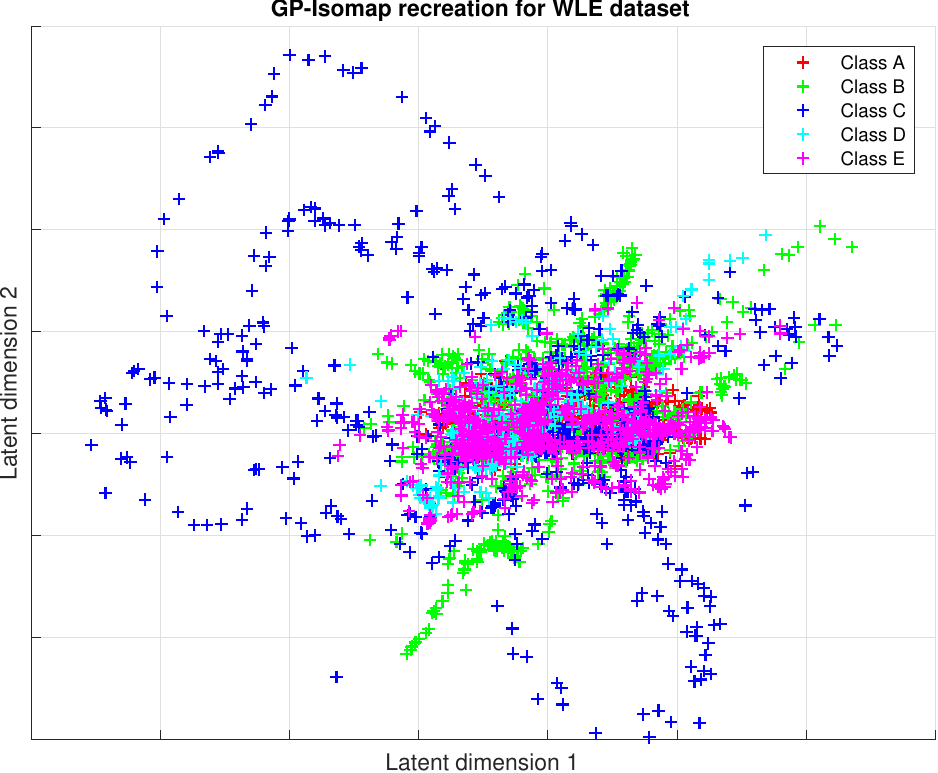}
        \caption{GP-Isomap recreation}
    \end{subfigure}
    \caption{Comparing predictions for S-Isomap++ and GP-Isomap empirically for the Human Activity Recognition data set. We observe that the low-dimensional representations uncovered by both algorithms are exactly the same, given the well-behaved variance resulted in clean separation of modes.}
\label{fig:wle_pred_comp}
\end{figure}
\section{Conclusions}\label{sec:conclusion}
We have proposed a streaming Isomap algorithm (GP-Isomap) that can be used to learn non-linear low-dimensional representation of high-dimensional data arriving in a streaming fashion. We prove that using a GPR formulation to map incoming data instances onto an existing manifold is equivalent to using existing geometric strategies~\citep{schoeneman2017,mahapatra2017}. Moreover, by utilizing a small batch for exact learning of the Isomap as well as training the GPR model, the method scales linearly with the size of the stream, thereby ensuring its applicability for practical problems. Using the Bayesian inference of the GPR model allows us to estimate the variance associated with the mapping of the streaming instances. The variance is shown to be a strong indicator of changes in the underlying stream properties on a variety of data sets. By utilizing the variance, one can devise re-training strategies that can include expanding the batch data set. While in the experiments we have demonstrated the ability of GP-Isomap to detect shifts in the underlying distributions, the algorithm can also be used to detect gradual shifts, as illustrated in Figure~\ref{fig:expts}. While we have focused on Isomap algorithm in this paper, similar formulations can be applied for other NLDR methods such as LLE~\citep{roweis2000}, etc., and will be explored as future research.
\begin{acks}
This material is based in part upon work supported by the National Science Foundation under award numbers CNS - 1409551 and IIS - 1641475. Access to computing facilities were provided by University of Buffalo Center for Computational Research.
\end{acks}


\input{ref.bbl}

\appendix

\section{Supplementary Results}\label{app:lemmas}
\begin{lem}\label{lemma1}
The matrix exponential for ${\bf M}$ for rank$\big({\bf M}\big) = d$ and symmetric ${\bf M}$ is given by
\[
e^{\bf M} = {\bf I} + \sum\limits_{i=1}^{d} \big( e^{{\bm \lambda}_i} - 1 \big){{\bf q}_i}{{\bf q}_i^\top}
\] where $\{ {{\bm \lambda}}_i \}_{i = 1,2 \ldots d}$ are the $d$ largest eigenvalues of ${\bf M}$ and $\{ {\bf q}_i \}_{i = 1,2 \ldots d}$ are the corresponding eigenvectors such that ${{\bf q}_i^\top}{{\bf q}_j} = {\bm \delta}_{i,j}$.
\end{lem}
\begin{proof}
Let ${\bf M}$ be an $n \times n$ real matrix. The exponential $e^{\bf M}$ is given by
\[
e^{\bf M} = \sum\limits_{k=0}^{\infty} \frac{1}{{k}\,!} {\bf M}^{\bf k} = {\bf I} + \sum\limits_{k=1}^{\infty} \frac{1}{{k}\,!} {\bf M}^{\bf k}
\] where ${\bf I}$ is the identity. Real, symmetric ${\bf M}$ has real eigenvalues and mutually orthogonal eigenvectors i.e. ${\bf M} = \sum\limits_{i=1}^{n} {\bm \lambda}_{i}{{\bf q}_i}{{\bf q}_i^\top} \mbox{ where } \{ {\bm \lambda}_{i} \}_{i = 1 \ldots n} \mbox{ are real and } {{\bf q}_i^\top{{\bf q}_j} = {\bm \delta}_{i,j}}$. Given ${\bf M}$ has rank $d$, we have ${\bf M} = \sum\limits_{i=1}^{d} {\bm \lambda}_{i}{{\bf q}_i}{{\bf q}_i^\top}$.
\begin{eqnarray}
\begin{aligned}
e^{\bf M} & =  {\bf I} + \sum\limits_{i=1}^{\infty} \frac{1}{i\,!} {\bf M}^i \\
& =  {\bf I} + \frac{1}{1\,!} \big({{\bm \lambda}_1}{{\bf q}_1}{{\bf q}_1^\top} + {{\bm \lambda}_2}{{\bf q}_2}{{\bf q}_2^\top} + \ldots + {{\bm \lambda}_d}{{\bf q}_d}{{\bf q}_d^\top} \big) \\
& +  \frac{1}{2\,!}\big({{\bm \lambda}_1}{{\bf q}_1}{{\bf q}_1^\top} + {{\bm \lambda}_2}{{\bf q}_2}{{\bf q}_2^\top} + \ldots +{{\bm \lambda}_d}{{\bf q}_d}{{\bf q}_d^\top} \big)^2 + \ldots \\
& =  {\bf I} + \big( \frac{{{\bm \lambda}_1}}{1\,!} + \frac{{{\bm \lambda}_1^2}}{2\,!} + \ldots \big){{\bf q}_1}{{\bf q}_1^\top} + \big( \frac{{{\bm \lambda}_2}}{1\,!} + \frac{{{\bm \lambda}_2^2}}{2\,!} + \ldots \big){{\bf q}_2}{{\bf q}_2^\top} +  \ldots \\
& +  \big( \frac{{{\bm \lambda}_d}}{1\,!} + \frac{{{\bm \lambda}_d^2}}{2\,!} + \ldots \big){{\bf q}_d}{{\bf q}_d^\top}\\
& =  {\bf I} + \big( e^{{\bm \lambda}_1} - 1 \big){{\bf q}_1}{{\bf q}_1^\top} + \big( e^{{\bm \lambda}_2} - 1 \big){{\bf q}_2}{{\bf q}_2^\top} + 	\ldots + \big( e^{{\bm \lambda}_d} - 1 \big){{\bf q}_d}{{\bf q}_d^\top}\\
& =  {\bf I} + \sum\limits_{i=1}^{d} \big( e^{{\bm \lambda}_i} - 1 \big){{\bf q}_i}{{\bf q}_i^\top}
\label{eqn:a3}
\end{aligned}
\end{eqnarray}
\end{proof}
\begin{lem}\label{lemma2}
The inverse of the Gaussian kernel for rank$\big({\bf M}\big) = 1$ and symmetric ${\bf M}$ is given by
\[
{\big( {\bf K} + {{\bm \sigma}_n}^2{\bf I} \big)}^{-1} = {\bf \alpha}{\bf I} - \frac{{\bf \alpha}^2{\bf c_1}{\bf q_1}{\bf q_1^\top}}{1 + {\bf \alpha}{\bf c_1}}
\] where ${{\bf q}_1}$ is the first eigenvector of {\bf M} i.e. ${{\bf q}_1^\top}{{\bf q}_1} = 1$, ${{\bm \lambda}_1}$ is the corresponding eigenvalue and ${\bm \alpha} = \frac{1}{\big( 1 + {{\bm \sigma}_n}^2 \big)}$ and ${{\bf c}_1} = \big[ \exp{\left(-\frac{{\bm \lambda}_1}{2{\bm \ell}^2}\right)} - 1 \big]$.
\end{lem}
\begin{proof}
Using \eqref{eqn:k2} for $d = 1$, we have 
\begin{equation}
\begin{split}
{\big( {\bf K} + {{\bm \sigma}_n}^2{\bf I} \big)}^{-1}& = {\big( {\bf I} + \big[\exp{\left(-\frac{{\bm \lambda}_1}{2{\bm \ell}^2}\right)} - 1 \big]{{\bf q}_1}{{\bf q}_1^\top} + {{\bm \sigma}_n}^2{\bf I} \big)}^{-1} \\
& = {\big( \big( 1 + {{\bm \sigma}_n}^2 \big){\bf I} + \big[\exp{\left(-\frac{{\bm \lambda}_1}{2{\bm \ell}^2}\right)} - 1 \big]{{\bf q}_1}{{\bf q}_1^\top} \big)}^{-1}
\end{split}
\label{eqn:a4}
\end{equation}
Representing $\frac{1}{\big( 1 + {{\bm \sigma}_n}^2 \big)}$ as ${\bm \alpha}$ and $\big[\exp{\left(-\frac{{\bm \lambda}_1}{2{\bm \ell}^2}\right)} - 1 \big]$ as ${{\bf c}_1}$ and using $\big( 1 + {{\bm \sigma}_n}^2 \big){\bf I}$ as ${\bf A}$, ${{\bf c}_1}{{\bf q}_1}$ as ${\bf u}$ and ${{\bf q}_1}$ as ${\bf v}$ in  the Sherman-Morrison identity~\citep{press1992}, we have \begin{equation}
\begin{split}
{\big( {\bf K} + {{\bm \sigma}_n}^2{\bf I} \big)}^{-1}& = {\bm \alpha}{\bf I} - \frac{{\bm \alpha}{\bf I}{{\bf c}_1}{{\bf q}_1}{{\bf q}_1^\top}{\bm \alpha}{\bf I}}{1 + {\bm \alpha}{{\bf c}_1}}\\
& = {\bm \alpha}{\bf I} - \frac{{\bm \alpha}^2{{\bf c}_1}{{\bf q}_1}{{\bf q}_1^\top}}{1 + {\bm \alpha}{{\bf c}_1}}
\end{split}
\label{eqn:a5}
\end{equation}
\end{proof}
\begin{lem}\label{lemma3}
The inverse of the Gaussian kernel for  rank$\big({\bf M}\big) = d$ and symmetric ${\bf M}$ is given by
\[
{\big( {\bf K} + {{\bm \sigma}_n}^2{\bf I} \big)}^{-1} = {\bm \alpha}{\bf I} - {\bm \alpha}^2 \sum\limits_{i=1}^{d} \frac{{{\bf c}_i}{{\bf q}_i}{{\bf q}_i^\top}}{1 + {\bm \alpha}{{\bf c}_i}}
\] where $\{ {\bm \lambda}_i \}_{i = 1,2 \ldots d}$ are the $d$ largest eigenvalues of ${\bf M}$ and $\{ {\bf q}_i \}_{i = 1,2 \ldots d}$ are the corresponding eigenvectors such that ${{\bf q}_i^\top}{{\bf q}_j} = {\bm \delta}_{i,j}$.
\end{lem}
\begin{proof}
Using the result of previous lemma iteratively, we get the required result \begin{equation}
{\big( {\bf K} + {{\bm \sigma}_n}^2{\bf I} \big)}^{-1} = {\bm \alpha}{\bf I} - {\bm \alpha}^2 \sum\limits_{i=1}^{d} \frac{{{\bf c}_i}{{\bf q}_i}{{\bf q}_i^\top}}{1 + {\bm \alpha}{{\bf c}_i}}
\label{eqn:a6}
\end{equation}
where ${\bm \alpha} = \frac{1}{\big( 1 + {{\bm \sigma}_n}^2 \big)}$ and ${{\bf c}_i} = \big[ \exp{\left(-\frac{{\bm \lambda}_i}{2{\bm \ell}^2}\right)} - 1 \big]$.
\end{proof}
\begin{lem}\label{lemma4}
The solution for Gaussian Process regression system, for the scenario when rank$\big({\bf M}\big) = 1$ and for symmetric ${\bf M}$ is given by 
\[
{\big( {\bf K} + {{\bm \sigma}_n}^2{\bf I} \big)}^{-1}{\bf y} = \frac{{\bm \alpha}\sqrt{\bm \lambda}_1{\bf q}_1}{{1 + {\bm \alpha}{\bf c_1}}}
\]
\end{lem}
\begin{proof}
Assuming the intrinsic dimensionality of the low-dimensional manifold to be $1$ implies that the inverse of the Gaussian kernel is as defined as in \eqref{eqn:a5}. ${\bf y}$ is $\sqrt{\bm \lambda}_1{\bf q}_1$ in this case (refer Section~\ref{ssec:NLDR}). Thus we have
\begin{equation}
\begin{split}
{\big( {\bf K} + {{\bm \sigma}_n}^2{\bf I} \big)}^{-1}{\bf y}& = \big( {\bm \alpha}{\bf I} - \frac{{\bm \alpha}^2{{\bf c}_1}{{\bf q}_1}{{\bf q}_1^\top}}{1 + {\bm \alpha}{{\bf c}_1}} \big) \big(\sqrt{\bm \lambda}_1{\bf q}_1\big)\\
& = {\bm \alpha}\sqrt{\bm \lambda}_1{\bf q}_1 - \frac{{\bm \alpha}^2\sqrt{\bm \lambda}_1{{\bf c}_1}{\bf q}_1}{{1 + {\bm \alpha}{{\bf c}_1}}} \\
& = \frac{{\bm \alpha}\sqrt{\bm \lambda}_1{\bf q}_1}{{1 + {\bm \alpha}{{\bf c}_1}}}
\end{split}
\label{eqn:a7}
\end{equation}
\end{proof}
\begin{lem}\label{lemma5}
The solution for Gaussian Process regression system, for the scenario when rank$\big({\bf M}\big) = d$ and for symmetric ${\bf M}$ is given by 
\[
{\big( {\bf K} + {{\bm \sigma}_n}^2{\bf I} \big)}^{-1}{\bf y} = \{ \frac{{\bm \alpha}\sqrt{\bm \lambda}_1{\bf q}_1}{{1 + {\bm \alpha}{{\bf c}_1}}} \; \frac{{\bm \alpha}\sqrt{\bm \lambda}_2{\bf q}_2}{{1 + {\bm \alpha}{{\bf c}_2}}} \; \ldots \; \frac{{\bm \alpha}\sqrt{\bm \lambda}_d{\bf q}_d}{{1 + {\bm \alpha}{{\bf c}_d}}} \}
\]
\end{lem}
\begin{proof}
Assuming the intrinsic dimensionality of the low-dimensional manifold to be $d$ implies that the inverse of the Gaussian kernel is as defined as in~\eqref{eqn:a6}. ${\bf y}$ is $\{ \sqrt{\bm \lambda}_1{\bf q}_1 \; \sqrt{\bm \lambda}_2{\bf q}_2 \; \ldots \; \sqrt{\bm \lambda}_d{\bf q}_d \}$ in this case  (refer Section~\ref{ssec:NLDR}), where ${{\bf q}_i^\top}{{\bf q}_j} = {\bm \delta}_{i,j}$.
Each of the ${k}$ dimensions of ${\big( {\bf K} + {{\bm \sigma}_n}^2{\bf I} \big)}^{-1}{\bf y}$ can be processed independently, similar to the previous lemma. For the ${i}$\textsuperscript{th} dimension, we have,
\begin{equation}
\begin{split}
{\big( {\bf K} + {{\bm \sigma}_n}^2{\bf I} \big)}^{-1}{\bf y_i}& = \big( {\bm \alpha}{\bf I} - {\bm \alpha}^2 \sum\limits_{j=1}^{d} \frac{{{\bf c}_j}{{\bf q}_j}{{\bf q}_j^\top}}{1 + {\bm \alpha}{{\bf c}_j}} \big) \big(\sqrt{\bm \lambda}_i{\bf q}_i\big)\\
& = {\bm \alpha}\sqrt{\bm \lambda}_i{\bf q}_i - {\bm \alpha}^2 \sum\limits_{j=1}^{d} \frac{{{\bf c}_j}{{\bf q}_j}{{\bf q}_j^\top}{\bf q}_i\big(\sqrt{\bm \lambda}_i\big)}{1 + {\bm \alpha}{{\bf c}_j}} \\
& = {\bm \alpha}\sqrt{\bm \lambda}_i{\bf q}_i - \frac{{\bm \alpha}^2\sqrt{\bm \lambda}_i{{\bf c}_i}{\bf q}_i}{{1 + {\bm \alpha}{{\bf c}_i}}} \\
& = \frac{{\bm \alpha}\sqrt{\bm \lambda}_i{\bf q}_i}{{1 + {\bm \alpha}{{\bf c}_i}}}
\end{split}
\label{eqn:a8}
\end{equation}
Thus we get the result,
\begin{equation}
{\big( {\bf K} + {\bf \sigma_n}^2{\bf I} \big)}^{-1}{\bf y} = \{ \frac{{\bm \alpha}\sqrt{\bm \lambda}_1{\bf q}_1}{{1 + {\bm \alpha}{{\bf c}_1}}} \; \frac{{\bm \alpha}\sqrt{\bm \lambda}_2{\bf q}_2}{{1 + {\bm \alpha}{{\bf c}_2}}} \; \ldots \; \frac{{\bm \alpha}\sqrt{\bm \lambda}_d{\bf q}_d}{{1 + {\bm \alpha}{{\bf c}_d}}} \}
\label{eqn:a9}
\end{equation}
\end{proof}

\end{document}